%% file: main.tex
\documentclass[sigconf]{aamas}

\usepackage{balance} 

\makeatletter
\gdef\@copyrightpermission{
  \begin{minipage}{0.2\columnwidth}
   \href{https://creativecommons.org/licenses/by/4.0/}{\includegraphics[width=0.90\textwidth]{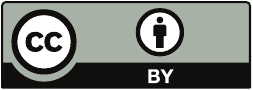}}
  \end{minipage}\hfill
  \begin{minipage}{0.8\columnwidth}
   \href{https://creativecommons.org/licenses/by/4.0/}{This work is licensed under a Creative Commons Attribution International 4.0 License.}
  \end{minipage}
  \vspace{5pt}
}
\makeatother

\setcopyright{ifaamas}
\acmConference[AAMAS '26]{Proc.\@ of the 25th International Conference
on Autonomous Agents and Multiagent Systems (AAMAS 2026)}{May 25 -- 29, 2026}
{Paphos, Cyprus}{C.~Amato, L.~Dennis, V.~Mascardi, J.~Thangarajah (eds.)}
\copyrightyear{2026}
\acmYear{2026}
\acmDOI{}
\acmPrice{}
\acmISBN{}


\input{macro}

\renewcommand{\cite}[1]{\citep{#1}}
\renewcommand{\citet}[1]{\cite{#1}}

\renewcommand{\t}{\text}
\newcommand{\op}[1]{\operatorname{#1}}
\newcommand{\C}[1]{{\mathcal{#1}}} 
\newcommand{\B}[1]{{\mathbb{#1}}} 
\newcommand{\BF}[1]{{\mathbf{#1}}} 

\newcommand{\bra}{\langle}
\newcommand{\ket}{\rangle}

\newcommand{\x}{\BF{x}}
\newcommand{\y}{\BF{y}}
\newcommand{\z}{\BF{z}}
\renewcommand{\vv}{\BF{v}}
\newcommand{\uu}{\BF{u}}
\newcommand{\w}{\BF{w}}
\newcommand{\oo}{\BF{o}}

\usepackage[algo2e,ruled]{algorithm2e}

\title{A Unified Framework for Analyzing Meta-algorithms in Online Convex Optimization}

\author{Mohammad Pedramfar}
\affiliation{
  \institution{Mila - Quebec AI Institute/McGill University}
  \city{Montreal, QC}
  \country{Canada}}
\email{mohammad.pedramfar@mila.quebec}

\author{Vaneet Aggarwal}
\affiliation{
  \institution{Purdue University}
  \city{West Lafayette, IN}
  \country{USA}}
\email{vaneet@purdue.edu}

\begin{abstract}
In this paper, we analyze the problem of online convex optimization in different settings, including different feedback types (full-information/semi-bandit/bandit/etc) in either stochastic or non-stochastic setting and different notions of regret (static adversarial regret/dynamic regret/adaptive regret).
This is done through a framework which allows us to systematically propose and analyze meta-algorithms for the various settings described above.
We show that any algorithm for online linear optimization with deterministic gradient feedback against fully adaptive adversaries is an algorithm for online convex optimization.
We also show that any such algorithm that requires full-information feedback may be transformed to an algorithm with semi-bandit feedback with comparable regret bound. We further show that algorithms that are designed for fully adaptive adversaries using deterministic semi-bandit feedback can obtain similar bounds using only stochastic semi-bandit feedback when facing oblivious adversaries. We use this to describe general meta-algorithms to convert first order algorithms to zeroth order algorithms with comparable regret bounds.
Our framework allows us to analyze online optimization in various settings, recovers several results in the literature with a simplified proof technique, and provides new results.
\end{abstract}

\newcommand{\BibTeX}{\rm B\kern-.05em{\sc i\kern-.025em b}\kern-.08em\TeX}

\begin{document}

\pagestyle{fancy}
\fancyhead{}

\maketitle

\input{contents}

\section{Acknowledgments}

The authors want to acknowledge the useful feedback on the paper and its ideas provided by Yuanyu Wan.

\newpage
\bibliographystyle{ACM-Reference-Format} 
\bibliography{references}
\newpage

\input{appendix}

\end{document}

%% file: macro.tex
\usepackage{algorithm}

\usepackage{amsfonts}

\usepackage{microtype}
\usepackage{graphicx}
\usepackage{booktabs} 

\usepackage{hyperref}

\usepackage{amsmath}
\usepackage{mathtools}
\usepackage{amsthm}
\usepackage[capitalize,noabbrev]{cleveref}

\usepackage{xcolor}
\usepackage{soul}

\theoremstyle{plain}
\newtheorem{theorem}{Theorem}

\newtheorem{lemma}{Lemma}
\newtheorem{corollary}{Corollary}
\theoremstyle{definition}
\newtheorem{definition}{Definition}

\theoremstyle{remark}
\newtheorem{remark}{Remark}

\usepackage{amsmath}
\usepackage{mathtools}
\usepackage{enumitem}
\usepackage{thm-restate}
\usepackage{amsfonts}
\usepackage{mathtools}
\usepackage{tcolorbox}
\usepackage{amsthm}
\usepackage{xcolor}
\usepackage{graphicx} 
\usepackage{algorithmic}
\newcount\Comments  
\newcommand{\kibitz}[2]{\ifnum\Comments=1{\textcolor{#1}{{#2}}}\fi}

\definecolor{darkgreen}{rgb}{0,0.6,0}

%% file: contents.tex
\section{Introduction}


Online optimization problems represent a class of problems where the decision-maker, referred to as an agent, aims to make sequential decisions in the face of an adversary with incomplete information \cite{shalev-shwartz12_onlin_learn_onlin_convex_optim,hazan2016introduction}. This setting mirrors a repeated game, where the agent and the adversary engage in a strategic interplay over a finite time horizon, commonly denoted as $T$. The dynamics of this game involve the agent's decision-making process, the adversary's strategy,  and the information exchange through a query oracle.

In this paper, we present a comprehensive approach to solve online convex optimization problems, addressing scenarios where the adversary can choose a function from the class of $\mu$--strongly convex functions (or convex functions for $\mu=0$) at each time step. 
Our approach encompasses different feedback types, including bandit, semi-bandit and full-information feedback in each iteration. 
It also encompasses different types of regret, such as adversarial regret, stochastic regret, and various forms of non-stationary regret.
While the problem has been studied in many special cases, this comprehensive approach sheds light on relations between different cases and the powers and limitations of some of the approaches used in different settings.

One of our key contribution lies in establishing that any semi-bandit feedback online linear (or quadratic) optimization algorithm for fully adaptive adversary is also an online convex (or strongly convex) optimization algorithm. 
While the above result is for semi-bandit feedback, we then show that in online convex optimization for fully adaptive adversary, semi-bandit feedback is generally enough and more information is not needed. 
Further,  we show that algorithms that are designed for fully adaptive adversaries using deterministic semi-bandit feedback can obtain similar bounds using only stochastic semi-bandit feedback when facing oblivious adversaries. 
Finally, we introduce meta algorithms, based on the variant of classical result of~\cite{flaxman2005online,shamir17_optim_algor_bandit_zero_order}, that convert algorithms for semi-bandit feedback, such as the ones mentioned above, to algorithms for bandit or zeroth-order feedback. In addition to recovering many results in the literature with simplified proofs, we give new static regret guarantees for online strongly convex optimization and adaptive regret guarantees for convex optimization with deterministic zeroth order feedback.

\begin{figure*}
{
\centering
\caption{ Summary of the main results}
\resizebox{.9\textwidth}{!}{\includegraphics{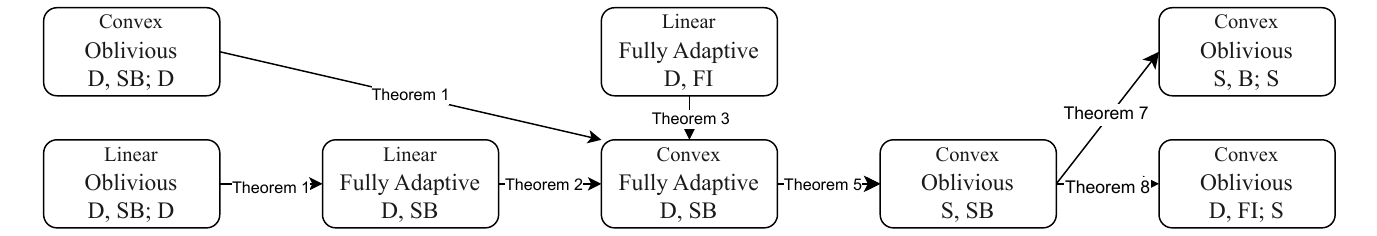}}
}
{ ~\\
Each arrow describes a procedure or meta-algorithm to convert an algorithm from one setting into another.
The first line describes the function class - linear/convex.
The second lines specifies the type of the adversary - fully adaptive or oblivious.
The last line starts with D or S, to denote deterministic or stochastic oracles.
After that, SB is used to denotes semi-bandit, B to denote bandit and FI to denote full information feedback.
Finally, the last D or S, if exists, is used to denotes deterministic or stochastic algorithm.
}\label{fig:main}
\end{figure*}
The key results of this paper are encapsulated in Figure~1, where each arrow represents a procedure or meta-algorithm that transforms an algorithm from one setting to another.
The primary contribution is three-fold:

\noindent (I) We provide a framework that simplifies analysis of meta-algorithms in online learning.

\noindent (II) Using this framework, we formulate several meta-algorithms that facilitate the conversion of algorithms across various settings.

\noindent (III) In particular, we convert several insights in the literature into precise theorems, for example the idea of using 1-point gradient estimator in online optimization is formalized in Meta-algorithms~\ref{alg:first-order-to-zeroth-order} and~\ref{alg:semi-bandit-to-bandit} and Theorems~\ref{thm:first-order-to-zero-order} and~\ref{thm:semi-bandit-to-bandit}.
Similarly, the insight that ``knowing the value of each loss function at two points is almost as useful as knowing the value of each function everywhere'' (See~\cite{agarwal2010optimal}) is formalized in Meta-algorithm~\ref{alg:first-order-to-det-zeroth-order} and Theorem~\ref{thm:first-order-to-det-zero-order}.

This framework not only simplifies proof analysis for numerous existing problems in the field (see Section~\ref{sec:applications}) but also provides a tool for easier extension of multiple results in the literature.
In the following, we go over some of the main meta-algorithms that we discuss in this paper.

\textbf{(i)} The first key result in this paper is Theorem~\ref{thm:main}. 
This result shows that any algorithm for online linear optimization with deterministic subgradient feedback results in a comparable regret when the function class is replaced by a class of convex functions. 
This is the key result that allows us to go from online linear optimization to convex optimization in Figure~1 with semi-bandit feedback. 
We note that this result recovers the result for online convex game in \cite{abernethy2008optimal}, where the authors analyzed online convex optimization with deterministic gradient feedback and fully adaptive adversary and showed that the min-max regret bounds of online convex optimization is equal to that of online linear optimization (or online quadratic optimization is all convex functions are $\mu$-strongly convex). 

While the above result is the key for going from online linear to online convex optimization, the main issue for its applicability is that it requires semi-bandit feedback. In the remaining results, we provide meta-algorithms to increase the applicability of this result. 

\textbf{(ii)} The second key result, given in Theorem \ref{thm:full-into-to-semi-bandit}, is that for any full-information algorithm for online convex optimization with fully adaptive adversary, there is a corresponding algorithm with regret bound of the same order which only requires semi-bandit feedback. This allows us to get results for semi-bandit feedback from that of full information. This combined with Theorem \ref{thm:main} allows us to go from any full-information algorithm for online linear optimization with fully adaptive adversary to an algorithm for online convex optimization with semi-bandit feedback and fully adaptive adversary. 
This combination has been marked as Theorem \ref{thm:full-into-to-semi-bandit} in Figure~1.

\textbf{(iii)} We note that the above results assume that the feedback oracle is deterministic. 
In the next key result, given in Theorem \ref{thm:det-to-stoch}, we relax this assumption. 
We show that any online optimization algorithm for fully adaptive adversaries using deterministic semi-bandit feedback obtains similar regret bounds when facing an oblivious adversary using only stochastic semi-bandit feedback. 
Thus, this result allows us to get the results for online convex optimization with stochastic semi-bandit feedback from the results with deterministic semi-bandit feedback, as can be seen in Fig.~1.
This result gives Theorem~6.5 in~\cite{hazan2016introduction} as a special case.
Further, this allows us to have a result that generalizes both adversarial regret and stochastic regret, as they are commonly defined in the literature.

\textbf{(iv)} The above discussion was mostly focused on full information and semi-bandit feedback settings.
However, in many practical setups, we only have access to zeroth order or bandit feedback. 
In the next key result, given in Theorems \ref{thm:first-order-to-zero-order}-\ref{thm:semi-bandit-to-bandit}, we use the smoothing trick \cite{nemirovsky83_probl_compl_method_effic_optim,flaxman2005online,pedramfar23_unified_approac_maxim_contin_dr_funct}, in order to estimate the gradient from the value oracle. 
Using this estimated gradient, the results for zeroth order and bandit feedback are derived from that for first order and semi-bandit feedback, respectively. 
These results apply to any optimization problem, not necessarily convex. 
This approach recovers the results in~\cite{flaxman2005online} with the base algorithm of Online Gradient Descent ($\mathtt{OGD}$). 
In particular, this approach also matches the SOTA dynamic and adaptive regret bounds for online convex optimization with bandit feedback \cite{zhao21_bandit_convex_optim_non_envir,garber22_new_projec_algor_onlin_convex,agarwal2010optimal} with a simplified proof.

\textbf{(v)} 
In the above results, we assumed that we only have access to an unbiased estimate of the zeroth order (or bandit) feedback.
In~\cite{agarwal2010optimal}, the setting of multi-point bandit feedback is introduced.
Even though their setting and notion of regret is different from our work, one insight is relevant.
Their analysis suggests that, in some cases, if we have access to exact values of the functions, knowing the value of each loss function at two points is almost as useful as knowing the value of each function everywhere.
We formalize this insight by designing a meta-algorithm based on the two-point gradient estimator~\cite{agarwal2010optimal,shamir17_optim_algor_bandit_zero_order} that allows us to convert algorithms for stochastic first order feedback into algorithms for deterministic zeroth order feedback with the same order of regret bound (Theorem~\ref{thm:first-order-to-det-zero-order}).
Note that this result holds for convex and non-convex optimization as it only relies on the function class being Lipschitz.
This result is shown to recover the result for dynamic regret for Improved Ader~\cite{zhao21_bandit_convex_optim_non_envir}, while achieving two new results for adaptive regret for convex function of $O(\sqrt{T})$ and static regret for strongly convex function of $O(\log T)$.

Our work sheds light on the relation between different problems, algorithms and techniques used in online optimization and provides general meta-algorithms that allow conversion between different problem settings. 
For instance, we can use the results for a deterministic algorithm for online linear optimization with oblivious adversaries and deterministic semi-bandit feedback to that for a stochastic algorithm for online convex optimization with oblivious adversaries and stochastic bandit feedback using the different arrows as in Fig.~1. 
The key technical novelty is the recognition that many results, previously understood only in very limited contexts, actually apply more broadly. 
This broader applicability not only simplifies the analysis of numerous existing results within the field but also enables the derivation of multiple new results. For instance, results in \cite{abernethy2008optimal,hazan2016introduction,flaxman2005online,zhao21_bandit_convex_optim_non_envir,agarwal2010optimal,shamir17_optim_algor_bandit_zero_order}  as mentioned above can be demonstrated as corollaries of these more generalized results. Further, we provide new results for deterministic zeroth-order feedback case, showing that adaptive regret for convex function is $O(\sqrt{T})$ and static regret for strongly convex function is $O(\log T)$.

{\bf Related works: }\label{sec:related} The term ``online convex optimization" was first defined in~\cite{zinkevich03_onlin}, where Online Gradient Descent ($\mathtt{OGD}$) was proposed, while this setting was first introduced in~\cite{gordon99_regret}.
The Follow-The-Perturbed-Leader ($\mathtt{FTPL}$) for online linear optimization was introduced by~\cite{kalai05_effic}, inspired by Follow-The-Leader algorithm originally suggested by~\cite{hannan57}.
Follow-The-Regularized-Leader ($\mathtt{FTRL}$) in the context of online convex optimization was coined in~\cite{shalev-shwartz07,shalev-shwartz07_onlin}.
The equivalence of $\mathtt{FTRL}$ and Online Mirror Descent ($\mathtt{OMD}$) was observed by~\cite{hazan10_extrac}. 
Non-stationary measures of regret have been studied in ~\cite{zinkevich03_onlin,hazan09_effic,besbes15_non_station_stoch_optim,daniely15_stron_adapt_onlin_learn,zhang18_dynam_regret_stron_adapt_method,zhang18_adapt_onlin_learn_dynam_envir,zhao20_dynam_regret_convex_smoot_funct,zhao21_bandit_convex_optim_non_envir,lu23_projec_adapt_regret_member_oracl,wang24_non_projec_free_onlin_learn,garber22_new_projec_algor_onlin_convex} 
and many other papers.
We refer the reader to surveys \cite{shalev-shwartz12_onlin_learn_onlin_convex_optim} and~\cite{hazan2016introduction} for more details.

\section{Background and Notations}

For a set $\C{S} \subseteq \B{R}^d$, we define its \textit{affine hull}  $\op{aff}(\C{S})$ to be the set of all points of the form $\sum_{i=1}^m \alpha_i \x_i$ for integer $m \geq 1$, points $(\x_i)_{i=1}^m$ in $\C{S}$, and real numbers $(\alpha_i)_{i=1}^m$.
The \textit{relative interior} of $\C{S}$ is defined as
\begin{align*}
\op{relint}(\C{S}) := \{ \x \in \C{S} \mid \exists r > 0, \B{B}_r(\x) \cap \op{aff}(\C{S}) \subseteq \C{S} \}.
\end{align*}
We use $\B{B}_r(\C{S})$ to denote the union of balls of radius $r$ centered at $\C{S}$ and use $\op{diam}(\C{S}) := \sup_{\x, \y \in \C{S}} \| \x - \y \|$ to denote the diameter of $\C{S}$.
A \textit{function class} is simply a set of real-valued functions.
Given a set $\C{S}$, a \textit{function class} over $\C{S}$ is simply a subset of all real-valued functions on $\C{S}$.
Given a real number $\mu \geq 0$ and a set $\C{S}$, we use $\BF{Q}_\mu$ to denote the class of functions of the form
$q_{\x, \oo}(\y) = \bra \oo, \y - \x \ket + \frac{\mu}{2} \| \y - \x \|^2$,
where $\oo \in \B{R}^d$ and $\x \in \C{S}$.
A set $\C{K} \subseteq \B{R}^d$ is called a \textit{convex set} if for all $\x,\y \in \C{K}$ and $\alpha \in [0, 1]$, we have $\alpha \x + (1 - \alpha) \y \in \C{K}$.
Given a convex set $\C{K}$, a function $f : \C{K} \to \B{R}$ is called \textit{convex} if, for all $\x, \y \in \C{K}$ and $\alpha \in [0, 1]$, we have
$f(\alpha \x + (1 - \alpha) \y) \leq \alpha f(\x) + (1 - \alpha) f(\y)$.
All convex functions are continuous on any point in the relative interior of their domains, but not necessarily at their relative boundary.
In this work, we will only focus on continuous functions.
If $\x \in \op{relint}(\C{K})$ and $f$ is convex and is differentiable at $\x$, then we have
$f(\y) - f(\x) \geq \bra \nabla f(\x), \y - \x \ket$, for all $\y \in \C{K}$.
A vector $\oo \in \B{R}^d$ is called a \textit{subgradient} of $f$ at $\x$ if 
$f(\y) - f(\x) \geq \bra \oo, \y - \x \ket$, for all $\y \in \C{K}$.
%
More generally, given $\mu \geq 0$, we say a vector $\oo \in \B{R}^d$ is a \textit{$\mu$-subgradient} of $f$ at $\x$ if 
$f(\y) - f(\x) \geq \bra \oo, \y - \x \ket + \frac{\mu}{2} \| \y - \x \|^2$,
for all $\y \in \C{K}$.
Given a convex set $\C{K}$, a Lipschitz continuous function $f : \C{K} \to \B{R}$ is convex (resp. $\mu$-strongly convex) if and only if it has a ($\mu$-)subgradient at all points $\x \in \C{K}$.
We use $\nabla^* f$ to denote the set of $\mu$-subgradients of $f$.

\section{Problem Setup}

Online optimization problems can be formalized as a repeated game between an agent and an adversary.
The game lasts for $T$ rounds on a convex domain $\C{K}$ where $T$ and $\C{K}$ are known to both players.
In $t$-th round, the agent chooses an action $\x_t$ from an action set $\C{K} \subseteq \B{R}^d$, then the adversary chooses a loss function $f_t \in \BF{F}$ and a query oracle for the function $f_t$.
Then, for $1 \leq i \leq k_t$, the agent chooses a points $\y_{t, i}$ and receives the output of the query oracle.
Here $k_t$ denotes the total number of queries made by the agent at time-step $t$, which may or may not be known in advance.

To be more precise, an agent consists of a tuple $(\Omega^\C{A}, \C{A}^{\t{action}}, \C{A}^{\t{query}})$, where $\Omega^\C{A}$ is a probability space that captures all the randomness of $\C{A}$.
We assume that, before the first action, the agent samples $\omega \in \Omega$.
The next element in the tuple, $\C{A}^{\t{action}} = (\C{A}^{\t{action}}_1, \cdots, \C{A}^{\t{action}}_T)$ is a sequence of functions such that $\C{A}_t$ that maps the history $\Omega^\C{A} \times \C{K}^{t-1} \times \prod_{s = 1}^{t-1} (\C{K} \times \C{O})^{k_s}$ to $\x_t \in \C{K}$ where we use $\C{O}$ to denote range of the query oracle.
The last element in the tuple, $\C{A}^{\t{query}}$, is the query policy.
For each $1 \leq t \leq T$ and $1 \leq i \leq k_t$, $\C{A}^{\t{query}}_{t, i} : \Omega^\C{A} \times \C{K}^t \times \prod_{s = 1}^{t-1} (\C{K} \times \C{O})^{k_s} \times (\C{K} \times \C{O})^{i-1}$ is a function that, given previous actions and observations, either selects a point $\y_t^i \in \C{K}$, i.e., query, or signals that the query policy at this time-step is terminated.
We may drop $\omega$ as one of the inputs of the above functions when there is no ambiguity.
We say the agent query function is \textit{trivial} if $k_t = 1$ and $\y_{t, 1} = \x_t$ for all $1 \leq t \leq T$.
In this case, we simplify the notation and use the notation $\C{A} = \C{A}^{\t{action}} = (\C{A}_1, \cdots, \C{A}_T)$ to denote the agent action functions and assume that the domain of $\C{A}_t$ is $\Omega^\C{A} \times (\C{K} \times \C{O})^{t-1}$.

A query oracle is a function that provides the observation to the agent.
Formally, a query oracle for a function $f$ is a map $\C{Q}$ defined on $\C{K}$ such that for each $\x \in \C{K}$, the $\C{Q}(\x)$ is a random variable taking value in the observation space $\C{O}$.
The query oracle is called a \textit{stochastic value oracle} or \textit{stochastic zeroth order oracle} if $\C{O} = \B{R}$ and $f(\x) = \B{E}[\C{Q}(\x)]$.
Similarly, it is called a \textit{stochastic (sub)gradient oracle} or \textit{stochastic first order oracle} if $\C{O} = \B{R}^d$ and $\B{E}[\C{Q}(\x)]$ is a (sub)gradient of $f$ at $\x$.
In all cases, if the random variable takes a single value with probability one, we refer to it as a \textit{deterministic} oracle.
Note that, given a function, there is at most a single deterministic gradient oracle, but there may be many deterministic subgradient oracles.
We will use $\nabla$ to denote the deterministic gradient oracle.
We say an oracle is bounded by $B$ if its output is always within the Euclidean ball of radius $B$ centered at the origin.
We say the agent takes \textit{semi-bandit feedback} if the oracle is first-order and the agent query function is trivial.
Similarly, it takes \textit{bandit feedback} if the oracle is zeroth-order and the agent query function is trivial
\footnote{This is a slight generalization of the common use of the term bandit feedback. Usually, bandit feedback refers to the case where the oracle is a \textit{deterministic} zeroth-order oracle and the agent query function is trivial, while our definition allows for \textit{stochastic} oracles.}.
If the agent query function is non-trivial, then we say the agent requires \textit{full-information feedback}.

An adversary $\op{Adv}$ is a set such that each element $\C{B} \in \op{Adv}$, referred to as a \textit{realized adversary}, is a sequence $(\C{B}_1, \cdots, \C{B}_T)$ of functions where each $\C{B}_t$ maps a tuple $(\x_1, \cdots, \x_t) \in \C{K}^t$ to a tuple $(f_t, \C{Q}_t)$ where $f_t \in \BF{F}$ and $\C{Q}_t$ is a query oracle for $f_t$ \footnote{Note that we do not assign a probability to each realized adversary since the notion of regret simply computes the supremum over all realizations.}. 
We say an adversary $\op{Adv}$ is \textit{oblivious} if for any realization $\C{B} = (\C{B}_1, \cdots, \C{B}_T)$, all functions $\C{B}_t$ are constant, i.e., they are independent of $(\x_1, \cdots, \x_t)$.
In this case, a realized adversary may be simply represented by a sequence of functions $(f_1, \cdots, f_T) \in \BF{F}^T$ and a sequence of query oracles $(\C{Q}_1, \cdots, \C{Q}_T)$ for these functions.
In this work we also consider adversaries that are \textit{fully adaptive}, i.e., adversaries with no restriction.
\footnote{
Another form of adversary considered in literature is a \textit{weakly adaptive} adversary where each function $\C{B}_t$ described above does not depend on $\x_t$ and therefore may be represented as a map defined on $\C{K}^{t-1}$.
Clearly any oblivious adversary is a weakly adaptive adversary and any weakly adaptive adversary is a fully adaptive adversary.}
Given a function class $\BF{F}$ and $i \in \{0, 1\}$, we use $\op{Adv}^\t{f}_i(\BF{F})$ to denote the set of all possible realized adversaries with deterministic $i$-th order oracles.
If the oracle is instead stochastic and bounded by $B$, we use $\op{Adv}^\t{f}_i(\BF{F}, B)$ to denote such an adversary.
Finally, we use $\op{Adv}^\t{o}_i(\BF{F})$ and $\op{Adv}^\t{o}_i(\BF{F}, B)$ to denote all oblivious realized adversaries with $i$-th order deterministic and stochastic oracles, respectively.

In order to handle different notions of regret with the same approach, for an agent $\C{A}$, adversary $\op{Adv}$, compact set $\C{U} \subseteq \C{K}^T$ and $1 \leq a \leq b \leq T$, we define \textit{regret} as
\begin{align*}
\C{R}_{\op{Adv}}^{\C{A}}(\C{U})[a, b] 
&:= \sup_{\C{B} \in \op{Adv}} \B{E} \left[ \sum_{t = a}^b f_t(\x_t) - \min_{\uu \in \C{U}} \sum_{t = a}^b f_t(\uu_t) \right]
\end{align*}
where $\uu = (\uu_1, \cdots, \uu_T)$ and the expectation in the definition of the regret is over the randomness of the algorithm and the query oracles.
We use the notation $\C{R}_{\C{B}}^{\C{A}}(\C{U})[a, b] := \C{R}_{\op{Adv}}^{\C{A}}(\C{U})[a, b]$ when $\op{Adv} = \{\C{B}\}$ is a singleton.

\textit{Static adversarial regret} or simply \textit{adversarial regret} corresponds to $a = 1$, $b = T$ and $\C{U} = \C{K}_{\star}^T := \{(\x, \cdots, \x) \mid \x \in \C{K}\}$.
When $a = 1$, $b = T$ and $\C{U}$ contains only a single element then it is referred to as the \textit{dynamic regret} \cite{zinkevich03_onlin,zhang18_adapt_onlin_learn_dynam_envir}. 
This notion of regret allows us to provide an upper bound on regret that may depend on the comparator. 
In Section~\ref{sec:applications}, we will discuss algorithms that provide guarantees depending on the path-length, defined with $P_T := \uu \mapsto \sum_{t = 1}^{T-1} \| \uu_t - \uu_{t+1} \| : \C{K}^T \to \B{R}$.
\textit{Adaptive regret}, is defined as
$\max_{1 \leq a \leq b \leq T} \C{R}_{\op{Adv}}^{\C{A}}(\C{K}_{\star}^T)[a, b]$ \cite{hazan09_effic}. 
A crucial point in this definition is that the comparison is with respect to a (potentially) different optimum for any interval.
We drop $a$, $b$ and $\C{U}$ when the statement is independent of their value or their value is clear from the context.

\begin{remark}
As a special case of static adversarial regret, if $\op{Adv}$ is oblivious and every $\C{B} \in \op{Adv}$ corresponds to $f_1 = f_2 = \cdots = f_T = f$ and $\C{Q}_1 = \C{Q}_2 \cdots = \C{Q}_T = \C{Q}$ for some $f \in \BF{F}$ and some stochastic oracle $\C{Q}$ for $f$, then the regret is referred to as the \textit{stochastic regret}.
\end{remark}

Another metric for comparing algorithms is \textit{high probability regret bounds} which we define for stochastic algorithms.
Given $\omega \in \Omega$, we may consider $\C{A}_\omega$ as a deterministic algorithm.
We say $h : [0, 1] \to [0, \infty)$ is a \textit{high probability regret bound} for $(\C{A}, \op{Adv})$ if for each $\C{B} \in \op{Adv}$ and $\delta \in [0, 1]$, we have
\[
\B{P}\left( \{ \omega \in \Omega \mid \C{R}_{\C{B}}^{\C{A}_\omega} \leq h(\delta) \} \right) \geq 1 - \delta.
\]
Note that the infimum of any family of high probability regret bounds is a high probability regret bound.
Hence we use $\overline{\C{R}}_{\op{Adv}}^{\C{A}} : [0, 1] \to [0, \infty)$ to denote the smallest high probability regret bound for $(\C{A}, \op{Adv})$.
In other words, $\overline{\C{R}}_{\C{B}}^{\C{A}}$ is the quantile function for the random variable $\omega \mapsto \C{R}_{\C{B}}^{\C{A}_\omega}$.


\section{Re-statement of Previous Result: Oblivious to fully adaptive adversary}\label{sec:oblivious-to-adaptive}

The following theorem states a result that is well-known in the literature.
For example, \cite{hazan2016introduction} only defined the notion of adaptive adversary when discussing stochastic algorithms and mentions that the results of all previous chapters, which were focused on deterministic algorithms, apply to any type of adversary.
Here we explicitly mention it as a theorem for completion.

\begin{theorem}\label{thm:det-is-adaptive}
Let $i \in \{0, 1\}$ and assume $\C{A}$ is a deterministic online algorithm designed for $i$-th order feedback and $\BF{F}$ is a function class.
Then we have
\[
\C{R}_{\op{Adv}_i^{\t{f}}(\BF{F})}^{\C{A}} = \C{R}_{\op{Adv}_i^{\t{o}}(\BF{F})}^{\C{A}}.
\]
\end{theorem}
See Appendix~\ref{app:det-is-adaptive} for proof.

\begin{corollary}\label{cor:some-det-algos}
The $\mathtt{OGD}$, $\mathtt{FTRL}$, and $\mathtt{OMD}$ algorithms are deterministic and therefore have same regret bound in fully-adaptive setting as the oblivious setting.
\end{corollary}

\section{Linear to convex with fully adaptive adversary}

Here we show that any semi-bandit feedback online linear optimization algorithm for fully adaptive adversary is also an online convex optimization algorithm.
We start with a definition.

\begin{definition}
Let $\BF{F}$ be a function class over $\C{K}$ and let $\mu \geq 0$.
We define $\BF{F}_\mu$, namely \textit{$\mu$-quadratization of $\BF{F}$}, to be class of functions $q : \C{K} \to \B{R}$ such that there exists $f \in \BF{F}$, $\x \in \C{K}$, and $\oo \in \nabla^* f(\x)$ such that
\[
q(\y) = \bra \oo, \y - \x \ket + \frac{\mu}{2} \| \y - \x \|^2 \in \BF{Q}_\mu.
\]
When $\mu = 0$, we may also refer to $\BF{F}_\mu$ as the \textit{linearization of $\BF{F}$}.
We say $\BF{F}$ is closed under $\mu$-quadratization if $\BF{F} \supseteq \BF{F}_\mu$.
Similarly, for $B > 0$, we define $\BF{Q}_\mu[B]$ to be the the class of functions $q$ defined above where instead we have $\oo \in \B{B}_{B}(\BF{0})$.
\end{definition}

\begin{theorem}\label{thm:main}
Let $\C{K} \subseteq \B{R}^d$ be a convex set, let $\mu \geq 0$ and let $\C{A}$ be algorithm for online optimization with semi-bandit feedback.
If $\BF{F}$ be a $\mu$-strongly convex function class over $\C{K}$, then we have
\begin{align*}
\C{R}_{\op{Adv}_1^{\t{f}}(\BF{F})}^{\C{A}}
\leq \C{R}_{\op{Adv}_1^{\t{f}}(\BF{F}_\mu)}^{\C{A}},
\quad
\overline{\C{R}}_{\op{Adv}_1^{\t{f}}(\BF{F})}^{\C{A}}
\leq \overline{\C{R}}_{\op{Adv}_1^{\t{f}}(\BF{F}_\mu)}^{\C{A}}
\end{align*}
Moreover, if $\BF{F}$ is closed under $\mu$-quadratization, then we have equality.
\end{theorem}

See Appendix~\ref{app:main} for proof.

\begin{corollary}[Theorem~14 in~\cite{abernethy2008optimal}]\label{cor:min-max}
The min-max regret bounds of online linear optimization with deterministic feedback and online convex optimization with deterministic feedback are equal.
\end{corollary}
To see why this is true, note that any algorithm has the same performance in both settings.
So it follows that an optimal algorithm for one setting is optimal for the other and therefore the min-max regret bounds are equal.

\section{Full information feedback to semi-bandit feedback }

\SetAlgorithmName{Meta-algorithm}{}

\begin{algorithm2e} 
\SetKwInOut{Input}{Input}\DontPrintSemicolon
\caption{Full information to semi-bandit - $\mathtt{FTS}(\C{A})$}
\label{alg:full-into-to-semi-bandit}
\small
\Input{ horizon $T$, algorithm $\C{A} = (\C{A}^\t{action}, \C{A}^\t{query})$, strong convexity coefficient $\mu \geq 0$}
\For{$t = 1, 2, \dots, T$}{
Play the action $\x_t$ chosen by $\C{A}^\t{action}$ \;
Let $f_t$ be the function chosen by the adversary \;
Query the oracle at the point $\x_t$ to get $\oo_t$ \;
Let $q_t := \y \mapsto \bra \oo_t, \y - \x_t \ket + \frac{\mu}{2}\|\y - \x_t\|^2$ \;
\For{$i$ starting from 1, while $\C{A}^\t{query}$ is not terminated for this time-step}{
    Let $\y_{t, i}$ be the query chosen by $\C{A}^\t{query}$ \;
    Return $\nabla q_t(\y_{t, i}) = \oo_t + \mu (\y_{t, i} - \x_t)$ as the output of the query oracle to $\C{A}$ \;
}
}
\end{algorithm2e}
Any algorithm designed for semi-bandit setting may be trivially applied in the first-order full-information feedback setting.
In the following theorem, we show that in online convex optimization for fully adaptive adversary, semi-bandit feedback is generally enough.
Specifically, we show that an algorithm that requires full information feedback could be converted to an algorithm that only requires semi-bandit feedback with the same regret bounds.
The meta-algorithm that does this conversion is described in Meta-algorithm~\ref{alg:full-into-to-semi-bandit}.
Next we show that $\mathtt{FTS}(\C{A})$ always performs at least as good as $\C{A}$ when the oracle is deterministic.

\begin{theorem}\label{thm:full-into-to-semi-bandit}
Let $\C{K} \subseteq \B{R}^d$ be a convex set, let $\mu \geq 0$ and let $\C{A}$ be algorithm for online optimization with full information first order feedback.
If $\BF{F}$ is a $\mu$-strongly convex function class over $\C{K}$
, then we have
\begin{align*}
\C{R}_{\op{Adv}_1^{\t{f}}(\BF{F})}^{\C{A}'}
\leq
\C{R}_{\op{Adv}_1^{\t{f}}(\BF{F}_\mu)}^{\C{A}},
\end{align*}
where $\C{A}' = \mathtt{FTS}(\C{A})$ is Meta-algorithm~\ref{alg:full-into-to-semi-bandit}.
In particular, if $\BF{F}$ is closed under $\mu$-quadratization, then
\begin{align*}
\C{R}_{\op{Adv}_1^{\t{f}}(\BF{F})}^{\C{A}'}
\leq
\C{R}_{\op{Adv}_1^{\t{f}}(\BF{F})}^{\C{A}}.
\end{align*}
\end{theorem}

See Appendix~\ref{app:full-into-to-semi-bandit} for proof.
Note that when the base algorithm $\C{A}$ is semi-bandit, we have $\mathtt{FTS}(\C{A}) = \C{A}$ and the above theorem becomes trivial.

\section{Lipschitz to non-Lipschitz}

It is well-known that when a convex function is restricted to a domain smaller than its original domain, then the restricted function is Lipschitz (See Observation~3 in~\cite{flaxman2005online} and Lemma~\ref{lem:conv-almost-lip}). 
Here we use a shrinking method described in~\cite{pedramfar23_unified_approac_maxim_contin_dr_funct}.
We choose a point $\BF{c} \in \op{relint}(\C{K})$ and a real number $r > 0$ such that $\B{B}_r(\BF{c}) \cap \op{aff}(\C{K}) \subseteq \C{K}$.
Then, for any shrinking parameter $0 \leq \alpha < r$, we define
$\hat{\C{K}}_\alpha := (1 - \frac{\alpha}{r}) \C{K} + \frac{\alpha}{r} \BF{c}$.
We define $\BF{F}|_{\hat{\C{K}}_\alpha}$ to function class $\{ f|_{\hat{\C{K}}_\alpha} \mid f \in \BF{F} \}$.
Given an adversary $\op{Adv}$, we define $\op{Adv}|_{\hat{\C{K}}_\alpha}$ to be the adversary constructed by restricting the output of $\op{Adv}$ to the set $\hat{\C{K}}_\alpha$.
Recall that the the domain $\C{K}$ of the adversary is known to the agent $\C{A}$.
For an online algorithm $\C{A}$, we define $\C{A}|_{\hat{\C{K}}_\alpha}$ to be the online algorithm resulting from restricting the domain of $\C{A}$ to $\hat{\C{K}}_\alpha$.

\begin{theorem}\label{thm:lip-to-nonlip}
Let $\C{A}$ be an online algorithm and $\BF{F}$ be a convex function class bounded by $M_0$.
Also let $\C{U} \subseteq \C{K}^T$ be a compact set and $\hat{\C{U}} := (1 - \frac{\alpha}{r}) \C{U} + \frac{\alpha}{r} \BF{c}$.
Then, for any adversary $\op{Adv}$, we have
\begin{align*}
\C{R}_{\op{Adv}}^{\C{A}|_{\hat{\C{K}}_\alpha}}(\C{U})
&\leq \C{R}_{\op{Adv}_{\hat{\C{K}}_\alpha}}^{\C{A}}(\hat{\C{U}})
+ \frac{\alpha M_0 T}{r}, \\
\overline{\C{R}}_{\op{Adv}}^{\C{A}|_{\hat{\C{K}}_\alpha}}(\C{U})
&\leq \overline{\C{R}}_{\op{Adv}_{\hat{\C{K}}_\alpha}}^{\C{A}}(\hat{\C{U}})
+ \frac{\alpha M_0 T}{r}.
\end{align*}
\end{theorem}

See Appendix~\ref{app:lip-to-nonlip} for proof.

\section{Deterministic feedback to stochastic feedback}

In this section, we show that algorithms that are designed for fully adaptive adversaries using deterministic semi-bandit feedback can obtain similar bounds using only stochastic semi-bandit feedback when facing oblivious adversaries.
In particular, there is no need for any variance reduction method, such as momentum \cite{mokhtari20_stoch_condit_gradien_method,chen18_projec_free_onlin_optim_stoch_gradien,xie20_effic_projec_free_onlin_method}, as long as we know that the algorithm is designed for a fully adaptive adversary.

\begin{theorem}\label{thm:det-to-stoch}
Let $\C{K} \subseteq \B{R}^d$ be a convex set, let $\mu \geq 0$ and let $\C{A}$ be an algorithm for online optimization with semi-bandit feedback.
If $\BF{F}$ is an $M_1$-Lipschitz $\mu$-strongly convex function class over $\C{K}$ and $B_1 \geq M_1$, then we have
\begin{align*}
\C{R}_{\op{Adv}_1^\t{o}(\BF{F}, B_1)}^{\C{A}}
&\leq \C{R}_{\op{Adv}_1^{\t{f}}(\BF{Q}_\mu[B_1])}^{\C{A}}, \\
\overline{\C{R}}_{\op{Adv}_1^\t{o}(\BF{F}, B_1)}^{\C{A}}
&\leq \overline{\C{R}}_{\op{Adv}_1^{\t{f}}(\BF{Q}_\mu[B_1])}^{\C{A}}.
\end{align*}
\end{theorem}
See Appendix~\ref{app:det-to-stoch} for proof.
Theorem~6.5 in~\cite{hazan2016introduction} may be seen as a special case of the above theorem when $\C{U}$ is a singleton and $\C{A}$ is deterministic.

\section{First order feedback to zeroth order feedback}
\begin{algorithm2e} 
\SetKwInOut{Input}{Input}\DontPrintSemicolon
\caption{First order to zeroth order - $\mathtt{FOTZO}(\C{A})$}
\label{alg:first-order-to-zeroth-order}
\small
\Input{Shrunk domain $\hat{\C{K}}_\alpha$, Linear space $\C{L}_0$, smoothing parameter $\delta \leq \alpha$, horizon $T$, algorithm $\C{A}$}
Pass $\hat{\C{K}}_\alpha$ as the domain to $\C{A}$ \;
$k \gets \op{dim}(\C{L}_0)$ \;
\For{$t = 1, 2, \dots, T$}{
Play $\x_t$, where $\x_t$ is the action chosen by $\C{A}$ \;
Let $f_t$ be the function chosen by the adversary \;
\For{$i$ starting from 1, while $\C{A}^\t{query}$ is not terminated for this time-step}{
    Sample $\vv_{t, i} \in \B{S}^1 \cap \C{L}_0$ uniformly \;
    Let $\y_{t, i}$ be the query chosen by $\C{A}^\t{query}$ \;
    Let $o_{t, i}$ be the output of the query oracle at the point $\y_{t, i} + \delta \vv_{t, i}$ \;
    Pass $\frac{k}{\delta} o_{t,i} \vv_{t,i}$ as the oracle output to $\C{A}$ \;
}
}
\end{algorithm2e}
When we do not have access to a gradient oracle, we rely on samples from a value oracle to estimate the gradient.
The ``smoothing trick'' is a classical idea in convex optimization (See~\cite{nemirovsky83_probl_compl_method_effic_optim})
which was first used in online convex optimization in~\cite{flaxman2005online}.
This idea involves averaging through spherical sampling around a given point.
Here we use a variant that was introduced in~\cite{pedramfar23_unified_approac_maxim_contin_dr_funct} which does not require extra assumptions on the convex set $\C{K}$.

 For a function $f : \C{K} \to \B{R}$ defined on a convex set $\C{K} \subseteq \B{R}^d$ and a smoothing parameter $0 < \delta \leq \alpha$, its $\delta$-smoothed version $\hat{f}_\delta : \hat{\C{K}}_\alpha \to \B{R}$ is given as
\begin{align*}
\hat{f}_\delta(\x) 
&:= \B{E}_{\z \sim \B{B}_\delta(\x) \cap \op{aff}(\C{K})}[f(\z)] 
= \B{E}_{\vv \sim \B{B}_1(\BF{0}) \cap \C{L}_0}[f(\x + \delta \vv)],
\end{align*}

\begin{algorithm2e}
\SetKwInOut{Input}{Input}\DontPrintSemicolon
\caption{Semi-bandit to bandit - $\mathtt{STB}(\C{A})$}
\label{alg:semi-bandit-to-bandit}
\small
\Input{Shrunk domain $\hat{\C{K}}_\alpha$, Linear space $\C{L}_0$, smoothing parameter $\delta \leq \alpha$, horizon $T$, algorithm $\C{A}$}
Pass $\hat{\C{K}}_\alpha$ as the domain to $\C{A}$ \;
$k \gets \op{dim}(\C{L}_0)$ \;
\For{$t = 1, 2, \dots, T$}{
Sample $\vv_t \in \B{S}^1 \cap \C{L}_0$ uniformly \;
Play $\x_t + \delta \vv_t$, where $\x_t$ is the action chosen by $\C{A}$ \;
Let $f_t$ be the function chosen by the adversary \;
Let $o_t$ be the output of the value oracle \;
Pass $\frac{k}{\delta} o_t \vv_t$ as the oracle output to $\C{A}$ \;
}
\end{algorithm2e}

where $\C{L}_0 = \op{aff}(\C{K}) - \x$, for any $\x \in \C{K}$ is the linear space that is a translation of the affine hull of $\C{K}$ and $\vv$ is sampled uniformly at random from the $k = \op{dim}(\C{L}_0)$-dimensional ball $\B{B}_1(\BF{0}) \cap \C{L}_0$.
Thus, the function value $\hat{f}_\delta(\x)$ is obtained by ``averaging'' $f$ over a sliced ball of radius $\delta$ around $\x$.
For a function class $\BF{F}$ over $\C{K}$, we use $\hat{\BF{F}}_\delta$ to denote $\{ \hat{f}_\delta \mid f \in \BF{F} \}$.
We will drop the subscripts $\alpha$ and $\delta$ when there is no ambiguity.

\begin{theorem}\label{thm:first-order-to-zero-order}
Let $\BF{F}$ be a function class over a convex set $\C{K}$ and choose $\BF{c}$ and $r$ as described above and let $\delta \leq \alpha < r$.
Let $\C{U} \subseteq \C{K}^T$ be a compact set and let $\hat{\C{U}} = (1 - \frac{\alpha}{r}) \C{U} + \frac{\alpha}{r} \BF{c}$.
Assume $\C{A}$ is an algorithm for online optimization with first order feedback.
Then, if $\C{A}' = \mathtt{FOTZO}(\C{A})$ where $\mathtt{FOTZO}$ is described by Meta-algorithm~\ref{alg:first-order-to-zeroth-order}, we have:
\begin{itemize}[leftmargin=*]
\item if $\BF{F}$ is convex and bounded by $M_0$, we have
\begin{align*}
\C{R}_{\op{Adv}^\t{o}_0(\BF{F}, B_0)}^{\C{A}'}(\C{U})
&\leq \C{R}_{\op{Adv}^\t{o}_1(\hat{\BF{F}}, \frac{k}{\delta}B_0)}^{\C{A}}(\hat{\C{U}}) + \left( \frac{6 \delta}{\alpha} + \frac{\alpha}{r} \right) M_0 T.
\end{align*}
\item if $\BF{F}$ is $M_1$-Lipschitz (but not necessarily convex), then we have
\begin{align*}
\C{R}_{\op{Adv}^\t{o}_0(\BF{F}, B_0)}^{\C{A}'}(\C{U})
&\leq \C{R}_{\op{Adv}^\t{o}_1(\hat{\BF{F}}, \frac{k}{\delta}B_0)}^{\C{A}}(\hat{\C{U}}) + \left( 3 + \frac{2 D \alpha}{r \delta} \right) \delta M_1 T,
\end{align*}
where $D = \op{diam}(\C{K})$.
\end{itemize}
\end{theorem}

\begin{theorem}\label{thm:semi-bandit-to-bandit}
Under the assumptions of Theorem~\ref{thm:first-order-to-zero-order}, if we further assume that $\C{A}$ is semi-bandit, then the same regret bounds hold with $\C{A}' = \mathtt{STB}(\C{A})$, where $\mathtt{STB}$ is described by Meta-algorithm~\ref{alg:semi-bandit-to-bandit}.
\end{theorem}

\begin{algorithm2e}
\SetKwInOut{Input}{Input}\DontPrintSemicolon
\caption{First order to zeroth order with two-point gradient estimator - $\mathtt{FOTZO\textrm{-}2P}(\C{A})$}
\label{alg:first-order-to-det-zeroth-order}
\small
\Input{ Shrunk domain $\hat{\C{K}}_\delta$, Linear space $\C{L}_0$, smoothing parameter $\delta < r$, horizon $T$, algorithm $\C{A}$}
Pass $\hat{\C{K}}_\delta$ as the domain to $\C{A}$ \;
$k \gets \op{dim}(\C{L}_0)$ \;
\For{$t = 1, 2, \dots, T$}{
Play $\x_t$, where $\x_t$ is the action chosen by $\C{A}$ \;
Let $f_t$ be the function chosen by the adversary \;
\For{$i$ starting from 1, while $\C{A}^\t{query}$ is not terminated for this time-step}{
    Sample $\vv_{t, i} \in \B{S}^1 \cap \C{L}_0$ uniformly \;
    Let $\y_{t, i}$ be the query chosen by $\C{A}^\t{query}$ \;
    Query the deterministic oracle at the points $\y_{t, i} + \delta \vv_{t, i}$ and $\y_{t, i} - \delta \vv_{t, i}$ \;
    Pass $\frac{k}{2 \delta} \left( f_t(\y_{t, i} + \delta \vv_{t, i}) - f_t(\y_{t, i} - \delta \vv_{t, i}) \right) \vv_{t, i}$ as the oracle output to $\C{A}$ \;
}
}
\end{algorithm2e}
  See Appendix~\ref{app:first-order-to-zero-order} for proof.

\begin{remark}
While $\mathtt{STB}$ turns a semi-bandit algorithm into a bandit algorithm, $\mathtt{FOTZO}$ turns a general first order algorithm into a zeroth order algorithm.
However, we note that even when $\C{A}$ is semi-bandit, $\mathtt{FOTZO}(\C{A})$ is not a bandit algorithm.
Instead it is a full-information zeroth order algorithm that uses a single query per iteration, but its query is not the same as its action.
This is because $\mathtt{FOTZO}(\C{A})$ plays the same point $\x_t$ as the base algorithm $\C{A}$, but then adds a noise before querying according to $\C{A}^\t{query}$.
\end{remark}

\begin{remark}
We have shown in the previous sections that for online convex optimization, semi-bandit feedback is enough.
However, we have included the $\mathtt{FOTZO}$ meta-algorithm, which converts a full-information algorithm to another full-information algorithm, since the results in this section apply also to non-convex online optimization problems.
\end{remark}
    

\section{First order feedback to deterministic zeroth order feedback}

When we have access to a deterministic zeroth order feedback, we may use the two point gradient estimator~\cite{agarwal2010optimal,shamir17_optim_algor_bandit_zero_order} to significantly improve the results of the previous section.

\begin{theorem}\label{thm:first-order-to-det-zero-order}
Let $\BF{F}$ be an $M_1$-Lipschitz function class over a convex set $\C{K}$ and choose $\BF{c}$ and $r$ as described above and let $\delta < r$.
Let $\C{U} \subseteq \C{K}^T$ be a compact set and let $\hat{\C{U}} = (1 - \frac{\delta}{r}) \C{U} + \frac{\delta}{r} \BF{c}$.
Assume $\C{A}$ is an algorithm for online optimization with first order feedback.
Then, if $\C{A}' = \mathtt{FOTZO\textrm{-}2P}(\C{A})$ where $\mathtt{FOTZO\textrm{-}2P}$ is described by Meta-algorithm~\ref{alg:first-order-to-det-zeroth-order}, we have
\begin{align*}
\C{R}_{\op{Adv}^\t{o}_0(\BF{F})}^{\C{A}'}(\C{U})
\leq \C{R}_{\op{Adv}^\t{o}_1(\hat{\BF{F}}, k M_1)}^{\C{A}}(\hat{\C{U}})
+ \left( 2 + \frac{2 D}{r} \right) \delta M_1 T,
\end{align*}
where $D = \op{diam}(\C{K})$.
\end{theorem}
Note that here we may choose $\delta = T^{-1}$ to see that the order of regret of $\C{A}'$ remains the same as that of $\C{A}$.
See Appendix~\ref{app:first-order-to-det-zero-order} for proof.

\section{Applications}\label{sec:applications}

In this section, we discuss applications of meta-algorithms discussed above to some specific base algorithms.
Note that all the results stated here are immediate applications of our framework to existing results and no extra proof is needed.

\textbf{Online Gradient Descent: }
Recall that $\mathtt{OGD}$ algorithm~\cite{zinkevich03_onlin} is a deterministic algorithm designed for deterministic semi-bandit feedback that obtains a regret bound of $O\left( M_1 T^{1/2} \right)$ over $M_1$-Lipschitz convex class $\BF{F}$ and $O\left( M_1^2 \log(T) \right)$ over strongly convex $M_1$-Lipschitz convex class $\BF{F}'$ (see~\cite{hazan2016introduction}).
\footnote{Here we assume $\op{diam}(\C{K}) \leq 1$ for simplicity. General bounds follow similarly and may depend on $\op{diam}(\C{K})$ and $r$, as described in Theorems~\ref{thm:first-order-to-zero-order} and~\ref{thm:first-order-to-det-zero-order}.}
%
%
%
%
Since $\mathtt{OGD}$ is deterministic, using Theorem~\ref{thm:det-is-adaptive}, we see that
\begin{align*}
\C{R}_{\op{Adv}^\t{f}_1(\BF{F})}^{\mathtt{OGD}}(\C{K}_*^T)[1, T]
    &= O\left( M_1 T^{1/2} \right), \\
 \quad
\C{R}_{\op{Adv}^\t{f}_1(\BF{F}')}^{\mathtt{OGD}}(\C{K}_*^T)[1, T]
    &= O\left( M_1^2 \log(T) \right).
\end{align*}
Hence we may apply Theorems~\ref{thm:det-to-stoch} to see that, if we have access to stochastic gradient oracle bounded by $B_1 \geq M_1$, then we have
\begin{align*}
\C{R}_{\op{Adv}^\t{o}_1(\BF{F}, B_1)}^{\mathtt{OGD}}(\C{K}_*^T)[1, T]
    &= O\left( B_1 T^{1/2} \right), \\
 \quad
\C{R}_{\op{Adv}^\t{o}_1(\BF{F}', B_1)}^{\mathtt{OGD}}(\C{K}_*^T)[1, T]
    &= O\left( B_1^2 \log(T) \right).
\end{align*}
For convex functions, we may apply Theorem~\ref{thm:semi-bandit-to-bandit}, with $\alpha = \delta = T^{-1/4}$.
Thus, if we have access to stochastic value oracle bounded by $B_0 \geq M_1$, we have
\begin{align*}
\C{R}_{\op{Adv}^\t{o}_0(\BF{F}, B_0)}^{\mathtt{STB}(\mathtt{OGD})}(\C{K}_*^T)[1, T]
&\leq \C{R}_{\op{Adv}^\t{o}_1(\BF{F}, \frac{k}{\delta} B_0)}^{\mathtt{OGD}}(\hat{\C{K}}_*^T)[1, T] + O\left( \delta M_1 T \right) \\
&= O\left( \frac{k}{\delta} B_0 T^{1/2} + \delta M_1 T \right)
= O\left( B_0 T^{3/4} \right),
\end{align*}
which generalizes the result of~\cite{flaxman2005online} to allow the feedback of the value oracle to be stochastic.
Similarly, for strongly convex functions, we may apply Theorem~\ref{thm:semi-bandit-to-bandit}, with $\alpha = \delta = T^{-1/3} \log(T)^{1/3}$ to see that
\begin{align*}
\C{R}_{\op{Adv}^\t{o}_0(\BF{F}', B_0)}^{\mathtt{STB}(\mathtt{OGD})}(\C{K}_*^T)[1, T]
\leq \C{R}_{\op{Adv}^\t{o}_1(\BF{F}', \frac{k}{\delta} B_0)}^{\mathtt{OGD}}(\hat{\C{K}}_*^T)[1, T] + O\left( \delta M_1 T \right) \\
= O\left( \frac{k^2}{\delta^2} B_0^2 \log(T) + \delta M_1 T \right)
= O\left( B_0^2 T^{2/3} \log(T)^{1/3} \right).
\end{align*}
This provides a new algorithm for bandit strongly convex optimization, matching the result of~\cite{agarwal2010optimal}.
However, if we further assume that the functions are smooth, then the regret bound could be improved to $\widetilde{O}(\sqrt{T})$ (See~\cite{hazan14_bandit_convex_optim,ito20_optim_algor_bandit_convex_optim}).


On the other hand, using Theorem~\ref{thm:first-order-to-det-zero-order} on $\mathtt{OGD}$, we see that
\begin{align*}
\C{R}_{\op{Adv}^\t{o}_0(\BF{F})}^{\mathtt{FOTZO\textrm{-}2P}(\mathtt{OGD})}(\C{K}_*^T)[1, T]
    &= O\left( M_1 T^{1/2} \right), \\
 \quad
\C{R}_{\op{Adv}^\t{o}_0(\BF{F}')}^{\mathtt{FOTZO\textrm{-}2P}(\mathtt{OGD})}(\C{K}_*^T)[1, T]
    &= O\left( M_1^2 \log(T) \right),
\end{align*}
where the first result is proven in~\cite{shamir17_optim_algor_bandit_zero_order} and the second one is novel.

\textbf{Ader and Improved Ader:}
In~\cite{zhang18_adapt_onlin_learn_dynam_envir}, 
Algorithm~2, i.e., ``Ader'', is a full-information deterministic algorithm designed for deterministic first order feedback.
Algorithm~3, i.e., ``Improved Ader'', is simply the result of applying the $\mathtt{FTS}$ meta-algorithm to Ader.
We may apply Theorems~\ref{thm:det-to-stoch} and~\ref{thm:semi-bandit-to-bandit}, with $\alpha = \delta = T^{-1/4}$, to see that
\begin{align}\label{eq:bandit-IA}
\C{R}_{\op{Adv}^\t{o}_0(\BF{F}, B_0)}^{\mathtt{STB}(\mathtt{IA})}(\uu)[1, T]
    &= O\left( B_0 \left(1 + P_T(\uu)\right)^{1/2} T^{3/4} \right),
\end{align}
where $\uu \in \C{K}^T$ and $P_T := \uu \mapsto \sum_{t = 1}^{T-1} \| \uu_t - \uu_{t+1} \| : \C{K}^T \to \B{R}$.
This matches the SOTA result of~\cite{zhao21_bandit_convex_optim_non_envir} for dynamic regret in the bandit feedback setting and generalizes it to allow for stochastic feedback.
Moreover, by applying Theorem~\ref{thm:first-order-to-det-zero-order} to Improved Ader, we recover another result of~\cite{zhao21_bandit_convex_optim_non_envir}
\begin{align*}
\C{R}_{\op{Adv}^\t{o}_0(\BF{F})}^{\mathtt{FOTZO\textrm{-}2P}(\mathtt{IA})}(\uu)[1, T]
    &= O\left( M_1 \left(1 + P_T(\uu)\right)^{1/2} T^{1/2} \right).
\end{align*}


\textbf{Online Gradient Descent with Separation Oracles:}
Similarly, the $\mathtt{SO\textrm{-}OGD}$ algorithm~\cite{garber22_new_projec_algor_onlin_convex} is a deterministic algorithm designed for deterministic semi-bandit feedback obtains a adaptive regret bound of $O(M_1 T^{1/2})$ (Theorem~14 in~\cite{garber22_new_projec_algor_onlin_convex})
Hence we may apply Theorems~\ref{thm:det-to-stoch} and~\ref{thm:semi-bandit-to-bandit}, with $\alpha = \delta = T^{-1/4}$, to see that
\begin{align*} 
\max_{1 \leq a \leq b \leq T} \C{R}_{\op{Adv}^\t{o}_0(\BF{F}, B_0)}^{\mathtt{STB}(\mathtt{SO\textrm{-}OGD})}(\C{K}_*^T)[a, b]
    = O\left( B_0 T^{3/4} \right).
\end{align*}
This matches the result of Theorem~15 in~\cite{garber22_new_projec_algor_onlin_convex} and generalizes it to allow for stochastic feedback.
Moreover, using Theorem~\ref{thm:first-order-to-det-zero-order}, we see that 
\begin{align*}
\max_{1 \leq a \leq b \leq T} \C{R}_{\op{Adv}^\t{o}_0(\BF{F})}^{\mathtt{FOTZO\textrm{-}2P}(\mathtt{SO\textrm{-}OGD})}(\C{K}_*^T)[a, b]
    = O\left( M_1 T^{1/2} \right), 
\end{align*}
which is a novel result for adaptive regret of convex functions with deterministic zeroth order  feedback.

\begin{remark}
A careful review of Algorithms~6 and~7 in~\cite{garber22_new_projec_algor_onlin_convex}, together with their Lemmas~12 and~13 and Theorems~14 and~15 reveals that they added an extra layer of complexity to their base algorithm in order to be able to move from semi-bandit to bandit feedback.
Specifically, if we instead use our Theorem~7, we can just drop $\delta'$ from Algorithms~6 and~7 and Lemmas~12 and~13 (by setting it to zero) to obtain simpler and clearer algorithms, statements and proofs.
Then their Theorem~15 will be an immediate corollary of applying our Theorem~\ref{thm:semi-bandit-to-bandit} to their previous results and the 3 pages of proof in their Appendix I will not be needed.
\end{remark}


%
%

\section{Conclusions}

This paper presents a comprehensive framework for the analysis of meta-algorithms in online learning, offering a simplified approach to understanding and extending existing results in the field. By demonstrating the equivalence between online linear and convex optimization algorithms, particularly in the context of semi-bandit feedback, we provide a versatile tool for addressing a wide range of optimization problems. Moreover, our findings pave the way for the development of novel meta-algorithms, enabling the seamless conversion of algorithms across different settings. 
In addition to recovering many results in the literature with simplified proofs, for convex function with deterministic zeroth order feedback, we obtain an adaptive regret of $O(\sqrt{T})$ (with $\mathtt{FOTZO\textrm{-}2P}(\mathtt{SO\textrm{-}OGD})$) and static regret of $O(\log T)$ for strongly convex case (with $\mathtt{FOTZO\textrm{-}2P}(\mathtt{OGD})$).


This work opens various research opportunities to delve into the interplay between online linear and convex optimization algorithms. While it streamlines results across numerous scenarios, we anticipate that this approach will yield novel insights into some unresolved problems moving forward.

%% file: appendix.tex

\onecolumn
\appendix

\section{Proof of Theorem~\ref{thm:det-is-adaptive}}\label{app:det-is-adaptive}

\begin{proof}
Note that $\op{Adv}^{\t{o}}_i(\BF{F}) \subseteq \op{Adv}_i^{\t{f}}(\BF{F})$ which directly implies that 
\[
\C{R}_{\op{Adv}_i^{\t{f}}(\BF{F})}^{\C{A}} \geq \C{R}_{\op{Adv}_i^{\t{o}}(\BF{F})}^{\C{A}}.
\]
To prove the inequality in the other direction, let $\C{B} \in \op{Adv}_i^{\t{f}}(\BF{F})$.
The game between $(\C{A}, \C{B})$ is a deterministic game and only a single sequence of functions $(f_1, \cdots, f_T)$ could be generated by $\C{B}$, i.e., there is a unique $\C{B}' := (f_1, \cdots, f_T) \in \op{Adv}^{\t{o}}_i(\BF{F})$ such that the game played by $(\C{A}, \C{B})$ is identical to the game played by $(\C{A}, \C{B}')$.
Therefore, we have
\begin{align*}
\C{R}_{\op{Adv}_i^{\t{f}}(\BF{F})}^{\C{A}} 
&= \sup_{\C{B} \in \op{Adv}_i^{\t{f}}(\BF{F})} \C{R}^\C{A}_{\C{B}} 
= \sup_{\C{B} \in \op{Adv}_i^{\t{f}}(\BF{F})} \C{R}^\C{A}_{\C{B}'} 
\leq \sup_{\C{B}' \in \op{Adv}_i^{\t{o}}(\BF{F})} \C{R}^\C{A}_{\C{B}'} 
= \C{R}_{\op{Adv}_i^{\t{o}}(\BF{F})}^{\C{A}}.
\qedhere
\end{align*}
\end{proof}

\section{Proof of Theorem~\ref{thm:main}}\label{app:main}

\begin{proof}~

\textbf{Regret:}

Let $\oo_t$ denote the output of the subgradient query oracle at time-step $t$.
For any realization $\C{B} = (\C{B}_1, \cdots, \C{B}_T) \in \op{Adv}_1^{\t{f}}(\BF{F})$, we define $\C{B}'_t(\x_1, \cdots, \x_t)$ to be the tuple $(q_t, \nabla)$ where
\[
q_t 
:= \y \mapsto \bra \oo_t, \y - \x_t \ket + \frac{\mu}{2} \| \y - \x_t \|^2 \in \BF{Q}_\mu.
\]
We also define $\C{B}' := (\C{B}'_1, \cdots, \C{B}'_T)$.
Note that each $\C{B}'_t$ is a function of $\x_1, \cdots, \x_t$ and therefore it belongs to $\op{Adv}_1^{\t{f}}(\BF{F}_\mu)$.
Since the algorithm uses semi-bandit feedback, the sequence of random vectors $(\x_1, \cdots, \x_T)$ chosen by $\C{A}$ is identical between the game with $\C{B}$ and $\C{B}'$.
Therefore, for any $\y \in \C{K}$, we have
\begin{align*}
q_t(\x_t) - q_t(\y)
= \bra \oo_t, \x_t - \y \ket - \frac{\mu}{2} \| \x_t - \y \|^2 
\geq f_t(\x_t) - f_t(\y),
\end{align*}
where the last inequality follows from fact that $\oo_t$ is a subgradient of $f_t$ at $\x_t$.
Therefore, we have
\begin{align}\label{eq:main:1}
\max_{\uu \in \C{U}} \left( \sum_{t = a}^b f_t(\x_t) - \sum_{t = a}^b  f_t(\uu_t) \right)
\leq \max_{\uu \in \C{U}} \left( \sum_{t = a}^b q_t(\x_t) - \sum_{t = a}^b q_t(\uu_t) \right),
\end{align}
which implies that
\begin{align*}
\C{R}_{\op{Adv}_1^{\t{f}}(\BF{F})}^{\C{A}}
&= \sup_{\C{B} \in \op{Adv}_1^{\t{f}}(\BF{F})} \C{R}_{\C{B}}^{\C{A}}
= \sup_{\C{B} \in \op{Adv}_1^{\t{f}}(\BF{F})} \B{E} \left[ \max_{\uu \in \C{U}} \left( \sum_{t = a}^b f_t(\x_t) - \sum_{t = a}^b  f_t(\uu_t) \right) \right] \\
&\leq \sup_{\C{B} \in \op{Adv}_1^{\t{f}}(\BF{F})} \B{E} \left[  \max_{\uu \in \C{U}} \left( \sum_{t = a}^b q_t(\x_t) - \sum_{t = a}^b q_t(\uu_t) \right) \right] 
\leq \sup_{\C{B}' \in \op{Adv}_1^{\t{f}}(\BF{F}_\mu)} \C{R}_{\C{B}'}^{\C{A}}
= \C{R}_{\op{Adv}_1^{\t{f}}(\BF{F}_\mu)}^{\C{A}}
\end{align*}

To prove the second claim, we note that $\BF{F} \supseteq \BF{F}_\mu$ implies that 
$\op{Adv}_1^{\t{f}}(\BF{F}) \supseteq \op{Adv}_1^{\t{f}}(\BF{F}_\mu)$.
Therefore
\begin{align*}
\C{R}_{\op{Adv}_1^{\t{f}}(\BF{F})}^{\C{A}}
&= \sup_{\C{B} \in \op{Adv}_1^{\t{f}}(\BF{F})} \C{R}_{\C{B}, \nabla^*}^{\C{A}} 
\geq \sup_{\C{B} \in \op{Adv}_1^{\t{f}}(\BF{F}_\mu)} \C{R}_{\C{B}, \nabla^*}^{\C{A}} 
= \C{R}_{\op{Adv}_1^{\t{f}}(\BF{F}_\mu)}^{\C{A}}
\end{align*}

\textbf{High probability regret bound:}

Assume that $h : [0, 1] \to \B{R}$ is a high probability regret bound for $(\C{A}, \op{Adv}_1^{\t{f}}(\BF{F}_\mu))$.
Using Inequality~\ref{eq:main:1}, we see that
\begin{align*}
\C{R}_{\C{B}}^{\C{A}_\omega}
= \max_{\uu \in \C{U}} \left( \sum_{t = a}^b f_t(\x_t) - \sum_{t = a}^b  f_t(\uu_t) \right)
\leq \max_{\uu \in \C{U}} \left( \sum_{t = a}^b q_t(\x_t) - \sum_{t = a}^b q_t(\uu_t) \right)
= \C{R}_{\C{B}'}^{\C{A}_\omega},
\end{align*}
which immediately implies that
\begin{align*}
\B{P}\left( \{ \omega \in \Omega^\C{A} \mid \C{R}_{\C{B}}^{\C{A}_\omega} \leq h(\delta) \} \right)
\geq \B{P}\left( \{ \omega \in \Omega^\C{A} \mid \C{R}_{\C{B}'}^{\C{A}_\omega} \leq h(\delta) \} \right)
\geq 1 - \delta.
\end{align*}
Hence $h$ is a high probability regret bound for $(\C{A}, \op{Adv}_1^{\t{f}}(\BF{F}))$ as well.

Moreover, if $\BF{F} \supseteq \BF{F}_\mu$, then $\op{Adv}_1^{\t{f}}(\BF{F}) \supseteq \op{Adv}_1^{\t{f}}(\BF{F}_\mu)$ and therefore, for any function $h : [0, 1] \to [0, \infty)$, we have
\begin{align*}
\inf_{\C{B} \in \op{Adv}_1^{\t{f}}(\BF{F})} \B{P}\left( \{ \omega \in \Omega^\C{A} \mid \C{R}_{\C{B}}^{\C{A}_\omega} \leq h(\delta) \} \right)
&\leq \inf_{\C{B} \in \op{Adv}_1^{\t{f}}(\BF{F}_\mu)} \B{P}\left( \{ \omega \in \Omega^\C{A} \mid \C{R}_{\C{B}}^{\C{A}_\omega} \leq h(\delta) \} \right).
\end{align*}
Hence any high probability regret bound for $(\C{A}, \op{Adv}_1^{\t{f}}(\BF{F}))$ is one for $(\C{A}, \op{Adv}_1^{\t{f}}(\BF{F}_\mu))$ as well.
\end{proof}

\section{Proof of Theorem~\ref{thm:full-into-to-semi-bandit}}\label{app:full-into-to-semi-bandit}

\begin{proof}
First we note that for any adversary $\op{Adv}$ over $\BF{Q}_\mu$ with deterministic first order feedback, under the algorithms $\C{A}'$, the functions $q_t$ and $f_t$ are equal up to an additive constant.
Hence $\C{A}'$ and $\C{A}$ return the same sequence of points and we have
\begin{align*}
\C{R}_{\op{Adv}}^{\C{A}'}
&= \C{R}_{\op{Adv}}^{\C{A}}
\qquad \t{and} \qquad
\overline{\C{R}}_{\op{Adv}}^{\C{A}'}
= \overline{\C{R}}_{\op{Adv}}^{\C{A}}
\end{align*}
Therefore, using Theorem~\ref{thm:main} and the fact that $\BF{F}_\mu \subseteq \BF{Q}_\mu$, we see that
\begin{align*}
\C{R}_{\op{Adv}_1^{\t{f}}(\BF{F})}^{\C{A}'}
&\leq \C{R}_{\op{Adv}_1^{\t{f}}(\BF{F}_\mu)}^{\C{A}'}
= \C{R}_{\op{Adv}_1^{\t{f}}(\BF{F}_\mu)}^{\C{A}},
\end{align*}
and similarly
\begin{align*}
\overline{\C{R}}_{\op{Adv}_1^{\t{f}}(\BF{F})}^{\C{A}'}
\leq \overline{\C{R}}_{\op{Adv}_1^{\t{f}}(\BF{F}_\mu)}^{\C{A}}.
\end{align*}

As a special case, if $\BF{F}$ is closed under $\mu$-quadratization, then we may use the same argument in the proof of Theorem~\ref{thm:main} to see that
\begin{align*}
\C{R}_{\op{Adv}_1^{\t{f}}(\BF{F}_\mu)}^{\C{A}} 
\leq \C{R}_{\op{Adv}_1^{\t{f}}(\BF{F})}^{\C{A}}
\end{align*}
and
\begin{align*}
\overline{\C{R}}_{\op{Adv}_1^{\t{f}}(\BF{F}_\mu)}^{\C{A}} 
\leq \overline{\C{R}}_{\op{Adv}_1^{\t{f}}(\BF{F})}^{\C{A}}
\end{align*}
which completes the proof.
\end{proof}

\section{Proof of Theorem~\ref{thm:lip-to-nonlip}}\label{app:lip-to-nonlip}

\begin{proof}
First note that for any $\C{B} \in \op{Adv}$, the game played between $(\C{A}|_{\hat{\C{K}}_\alpha}, \C{B})$ is identical to the game played between $(\C{A}, \C{B}|_{\hat{\C{K}}_\alpha})$.
Let $\uu \in \op{argmin}_{\uu \in \C{U}} \sum_{t = a}^b f_t(\uu_t)$ and $\hat{\uu} \in \op{argmin}_{\uu \in \hat{\C{U}}} \sum_{t = a}^b f_t(\uu_t)$.
We have
\begin{align*}
\C{R}_{\C{B}}^{\C{A}|_{\hat{\C{K}}_\alpha}}
- \C{R}_{\C{B}|_{\hat{\C{K}}_\alpha}}^{\C{A}}
&= \B{E}\left[ \sum_{t = a}^b f_t(\x_t) - \sum_{t = a}^b f_t(\uu_t) \right]
- \B{E}\left[ \sum_{t = a}^b f_t(\x_t) - \sum_{t = a}^b f_t(\hat{\uu}_t) \right] \\
&= \B{E}\left[ \sum_{t = a}^b f_t(\hat{\uu}_t) - \sum_{t = a}^b f_t(\uu_t) \right].
\end{align*}
We have
\begin{align*}
\sum_{t = a}^b f_t(\hat{\uu}_t) 
&= \min_{\hat{\uu} \in \hat{\C{U}}} \sum_{t = a}^b f_t(\hat{\uu}_t) \\
&= \min_{\uu \in \C{U}} \sum_{t = a}^b f_t\left(\left(1 - \frac{\alpha}{r} \right) \uu_t + \frac{\alpha}{r} \BF{c} \right) 
\tag{Definition of $\hat{\C{U}}$}\\
&\leq \min_{\uu \in \C{U}} \sum_{t = a}^b \left(\left(1 - \frac{\alpha}{r} \right) f_t(\uu_t) + \frac{\alpha}{r} f_t(\BF{c}) \right) 
\tag{Convexity} \\
&\leq \min_{\uu \in \C{U}} \sum_{t = a}^b \left(\left(1 - \frac{\alpha}{r} \right) f_t(\uu_t) + \frac{\alpha}{r} M_0 \right) \\
&= \frac{\alpha M_0 T}{r} + \left(1 - \frac{\alpha}{r} \right) \min_{\uu \in \C{U}} \sum_{t = a}^b f_t(\uu_t) \\
&\leq \frac{\alpha M_0 T}{r} + \sum_{t = a}^b f_t(\uu_t).
\end{align*}
Putting these together, we see that
\begin{align}
\C{R}_{\C{B}}^{\C{A}|_{\hat{\C{K}}_\alpha}}(\C{U})
- \C{R}_{\C{B}|_{\hat{\C{K}}_\alpha}}^{\C{A}}(\hat{\C{U}})
&\leq \frac{\alpha M_0 T}{r}.
\label{eq:lip-to-nonlip}
\end{align}
Therefore, we have 
\begin{align*}
\C{R}_{\op{Adv}}^{\C{A}|_{\hat{\C{K}}_\alpha}}(\C{U})
&= \sup_{\C{B} \in \op{Adv}} \C{R}_{\C{B}}^{\C{A}|_{\hat{\C{K}}_\alpha}}(\C{U})
\leq \sup_{\C{B} \in \op{Adv}} \C{R}_{\C{B}|_{\hat{\C{K}}_\alpha}}^{\C{A}}(\hat{\C{U}})
+ \frac{\alpha M_0 T}{r} \\
&\leq \sup_{\C{B} \in \op{Adv}_{\hat{\C{K}}_\alpha}} \C{R}_{\C{B}}^{\C{A}}(\hat{\C{U}})
+ \frac{\alpha M_0 T}{r} 
= \C{R}_{\op{Adv}_{\hat{\C{K}}_\alpha}}^{\C{A}}(\hat{\C{U}})
+ \frac{\alpha M_0 T}{r}.
\end{align*}
The proof of the claim for high probability regret bounds follows from Equation~\ref{eq:lip-to-nonlip} and the same argument used in the last part of the proof of Theorem~\ref{thm:main}.
\end{proof}

\section{Proof of Theorem~\ref{thm:det-to-stoch}}\label{app:det-to-stoch}

\begin{proof}
Let $\Omega^\C{Q} = \Omega^\C{Q}_1 \times \cdots \times \Omega^\C{Q}_T$ capture all sources of randomness in the query oracles of $\op{Adv}_1^\t{o}(\BF{F}, B_1)$, i.e., for any choice of $\theta \in \Omega^\C{Q}$, the query oracle is deterministic.
Hence for any $\theta \in \Omega^\C{Q}$ and realized adversary $\C{B} \in \op{Adv}_1^\t{o}(\BF{F}, B_1)$, we may consider $\C{B}_\theta$ as an object similar to an adversary with a deterministic oracle.
However, note that $\C{B}_\theta$ does not satisfy the unbiasedness condition of the oracle, i.e., the returned value of the oracle is not necessarily the gradient of the function at that point.
Recall that $h_t$ denotes the history and $\C{B}_t$ maps a tuple $(h_t, \x_t)$ to a tuple of $f_t$ and a stochastic query oracle for $f_t$.
We will use $\B{E}_{\Omega^\C{Q}}$ to denote the expectation with respect to the randomness of query oracle and $\B{E}_{\Omega^\C{Q}_t}[\cdot] := \B{E}_{\Omega^\C{Q}}[\cdot | f_t, \x_t]$ to denote the expectation conditioned on the function and the action of the agent.
Similarly, let $\B{E}_{\Omega^\C{A}}$ denote the expectation with respect to the randomness of the agent.
Let $\oo_t$ be the random variable denoting the output of the subgradient query oracle at time-step $t$ and let 
\[
\bar{\oo}_t := \B{E}[\oo_t \mid f_t, \x_t] = \B{E}_{\Omega^\C{Q}_t}[\oo_t]
\]
be a subgradient of $f_t$ at $\x_t$.

For any realization $\C{B} = (f_1, \cdots, f_T) \in \op{Adv}^\t{o}(\BF{F})$ and any $\theta \in \Omega^\C{Q}$, we define $\C{B}'_{\theta, t}(\x_1, \cdots, \x_t) $ to be the pair $(q_t, \nabla)$ where
\[
q_t 
:= \y \mapsto \bra \oo_t, \y - \x_t \ket + \frac{\mu}{2} \| \y - \x_t \|^2 \in \BF{Q}_\mu.
\]
We also define $\C{B}'_{\theta} := (\C{B}'_{\theta, 1}, \cdots, \C{B}'_{\theta, T})$.
Note that a specific choice of $\theta$ is necessary to make sure that the function returned by $\C{B}'_{\theta, t}$ is a deterministic function of $\x_1, \cdots, \x_t$ and not a random variable and therefore $\C{B}'_{\theta}$ belongs to $\op{Adv}_1^{\t{f}}(\BF{Q}_\mu[B_1])$.

Since the algorithm uses (semi-)bandit feedback, given a specific value of $\theta$, the sequence of random vectors $(\x_1, \cdots, \x_T)$ chosen by $\C{A}$ is identical between the game with $\C{B}_\theta$ and $\C{B}'_\theta$.
Therefore, for any $\uu_t \in \C{K}$, we have
\begin{align*}
f_t(\x_t) - f_t(\uu_t) 
&\leq \bra \bar{\oo}_t, \x_t - \uu_t \ket - \frac{\mu}{2} \| \x_t - \uu_t \|^2 \\
&= \bra \B{E}\left[ \oo_t \mid h_t, \x_t \right], \x_t - \uu_t \ket - \frac{\mu}{2} \| \x_t - \uu_t \|^2 \\
&= \B{E}\left[ \bra \oo_t, \x_t - \uu_t \ket - \frac{\mu}{2} \| \x_t - \uu_t \|^2 \mid f_t, \x_t \right] \\
&= \B{E}\left[ q_t(\x_t) - q_t(\uu_t) \mid h_t, \x_t \right] \\
&= \B{E}_{\Omega^\C{Q}_t} \left[ q_t(\x_t) - q_t(\uu_t) \right]
\end{align*}
where the first inequality follows from the fact that $f_t$ is $\mu$-strongly convex and $\bar{\oo}_t$ is a subgradient of $f_t$ at $\x_t$.
Therefore we have
\begin{align*}
\B{E}_{\Omega^\C{Q}} \left[
\sum_{t = a}^b f_t(\x_t) - \sum_{t = a}^b  f_t(\uu_t)
\right]
\leq
\B{E}_{\Omega^\C{Q}} \left[
\sum_{t = a}^b \B{E}_{\Omega^\C{Q}_t} \left[ q_t(\x_t) - q_t(\uu_t) \right]
\right]
=
\B{E}_{\Omega^\C{Q}} \left[
\sum_{t = a}^b q_t(\x_t) - q_t(\uu_t)
\right].
\end{align*}
Since $\C{B}$ is oblivious, the sequence $(f_1, \cdots, f_T)$ is not affected by the randomness of query oracles.
Therefore we have
\begin{align*}
\C{R}_{\C{B}}^{\C{A}_\omega}
&= 
\B{E}_{\Omega^\C{Q}} \left[
\sum_{t = a}^b f_t(\x_t) - \min_{\uu \in \C{U}} \sum_{t = a}^b  f_t(\uu_t)
\right] \\
&= 
\max_{\uu \in \C{U}}
\B{E}_{\Omega^\C{Q}} \left[
\sum_{t = a}^b f_t(\x_t) - \sum_{t = a}^b  f_t(\uu_t)
\right] \\
&\leq
\max_{\uu \in \C{U}}
\B{E}_{\Omega^\C{Q}} \left[
\sum_{t = a}^b q_t(\x_t) - \sum_{t = a}^b q_t(\uu_t)
\right] \\
&\leq
\B{E}_{\Omega^\C{Q}} \left[
\max_{\uu \in \C{U}}
\left(
\sum_{t = a}^b q_t(\x_t) - \sum_{t = a}^b q_t(\uu_t)
\right)
\right] 
= \B{E}_{\Omega^\C{Q}} \left[
\C{R}_{\C{B}'_\theta}^{\C{A}_\omega}
\right],
\end{align*}
where the second inequality follows from Jensen's inequality.

\textbf{Regret:}

By taking the expectation with respect to $\omega \in \Omega^\C{A}$, we see that
\begin{align*}
\C{R}_{\C{B}}^{\C{A}}
= \B{E}_{\Omega^\C{A}} \left[
\C{R}_{\C{B}}^{\C{A}_\omega}
\right]
\leq 
\B{E}_{\Omega^\C{A}} \left[
\B{E}_{\Omega^\C{Q}} \left[
\C{R}_{\C{B}'_\theta}^{\C{A}_\omega}
\right]
\right]
= 
\B{E}_{\Omega^\C{Q}} \left[
\C{R}_{\C{B}'_\theta}^{\C{A}}
\right]
\leq
\sup_{\theta \in \Omega^\C{Q}} 
\C{R}_{\C{B}'_\theta}^{\C{A}}
\end{align*}
Hence we have
\begin{align*}
\C{R}_{\op{Adv}_1^\t{o}(\BF{F}, B_1)}^{\C{A}}
= 
\sup_{\C{B} \in \op{Adv}_1^\t{o}(\BF{F}, B_1)}
\C{R}_{\C{B}}^{\C{A}}
\leq 
\sup_{\C{B} \in \op{Adv}_1^\t{o}(\BF{F}, B_1), \theta \in \Omega^\C{Q}}
\C{R}_{\C{B}'_\theta}^{\C{A}}
\leq 
\sup_{\C{B}' \in \op{Adv}_1^{\t{f}}(\BF{Q}_\mu[B_1])}
\C{R}_{\C{B}'}^{\C{A}}
= 
\C{R}_{\op{Adv}_1^{\t{f}}(\BF{Q}_\mu[B_1])}^{\C{A}}
\end{align*}

\textbf{High probability regret bound:}

By taking the supremum over $\C{B} \in \op{Adv}_1^\t{o}(\BF{F}, B_1)$, we see that
\begin{align*}
\C{R}_{\op{Adv}_1^\t{o}(\BF{F}, B_1)}^{\C{A}_\omega}
&= 
\sup_{\C{B} \in \op{Adv}_1^\t{o}(\BF{F}, B_1)}
\C{R}_{\C{B}}^{\C{A}_\omega}
\leq 
\sup_{\C{B} \in \op{Adv}_1^\t{o}(\BF{F}, B_1)}
\B{E}_{\Omega^\C{Q}} \left[
\C{R}_{\C{B}'_\theta}^{\C{A}_\omega}\
\right] \\ 
&\leq 
\sup_{\C{B} \in \op{Adv}_1^\t{o}(\BF{F}, B_1), \theta \in \Omega^\C{Q}}
\C{R}_{\C{B}'_\theta}^{\C{A}_\omega}
\leq 
\sup_{\C{B}' \in \op{Adv}_1^{\t{f}}(\BF{Q}_\mu[B_1])}
\C{R}_{\C{B}'}^{\C{A}_\omega}
= 
\C{R}_{\op{Adv}_1^{\t{f}}(\BF{Q}_\mu[B_1])}^{\C{A}_\omega}
\end{align*}
Therefore we may use the same argument in the proof of Theorem~\ref{thm:main} to see that
\begin{align*}
\overline{\C{R}}_{\op{Adv}_1^\t{o}(\BF{F}, B_1)}^{\C{A}}
\leq
\overline{\C{R}}_{\op{Adv}_1^{\t{f}}(\BF{Q}_\mu[B_1])}^{\C{A}}
\end{align*}
\end{proof}

\section{Proof of Theorems~\ref{thm:first-order-to-zero-order} and~\ref{thm:semi-bandit-to-bandit}}\label{app:first-order-to-zero-order}

First we prove the following lemma which is a simple generalization of Observation~3 in~\cite{flaxman2005online} to the construction of $\hat{\C{K}}$ that we use here and is proved similarly.
\begin{lemma}\label{lem:conv-almost-lip}
For any $\x \in \hat{\C{K}}_\alpha$ and $\y \in \C{K}$ and any convex function $f : \C{K} \to [-M_0, M_0]$, we have
\begin{align*}
|f(\x) - f(\y)| \leq \frac{2 M_0}{\alpha} \| \x - \y \|
\end{align*}
\end{lemma}
\begin{proof}
Let $\y = \x + \Delta$.
If $\| \Delta \| \geq \alpha$, then we have
\begin{align*}
|f(\x) - f(\y)| \leq 2 M_0 \leq \frac{2 M_0}{\alpha} \| \x - \y \|.
\end{align*}
Otherwise, then let $\z = \x + \alpha \frac{\Delta}{\| \Delta \|}$.
According to Lemma~7 in~\cite{pedramfar23_unified_approac_maxim_contin_dr_funct}, we have $\z \in \C{K}$.
We have $\y = \frac{\| \Delta \|}{\alpha} \z + \left( 1 - \frac{\| \Delta \|}{\alpha} \right) \x$.
Therefore, since $f$ is convex and bounded by $M_0$, we have
\begin{align*}
f(\y) 
&\leq \frac{\| \Delta \|}{\alpha} f(\z) + \left( 1 - \frac{\| \Delta \|}{\alpha} \right) f(\x)
= f(\x) + \frac{\| \Delta \|}{\alpha} (f(\z) - f(\x))
\leq f(\x) + \frac{2 M_0}{\alpha} \| \x - \y \|.
\qedhere
\end{align*}
\end{proof}

We start with a proof of Theorem~\ref{thm:semi-bandit-to-bandit} and then adapt the proof to work for Theorem~\ref{thm:first-order-to-zero-order}.

\begin{proof}[Proof of Theorem~\ref{thm:semi-bandit-to-bandit}]

Note that any realized adversary $\C{B} \in \op{Adv}^\t{o}_0(\BF{F}, B_0)$ may be represented as a sequence of functions $(f_1, \cdots, f_T)$ and a corresponding sequence of query oracles $(\C{Q}_1, \cdots, \C{Q}_T)$.
For such realized adversary $\C{B}$, we define $\hat{B}$ to be the realized adversary corresponding to $(\hat{f}_1, \cdots, \hat{f}_T)$ with the stochastic gradient oracles
\begin{align}\label{eq:semi-bandit-to-bandit:1}
\hat{\C{Q}}_t(\x) := \frac{k}{\delta} \C{Q}_t(\x + \delta \vv) \vv,
\end{align}
where $\vv$ is a random vector, taking its values uniformly from $\B{S}^1 \cap \C{L}_0 = \B{S}^1 \cap (\op{aff}(\C{K}) - \z)$, for any $\z \in \C{K}$ and $k = \op{dim}(\C{L}_0)$.
Since $\C{Q}_t$ is a stochastic value oracle for $f_t$, according to Remark~4 in~\cite{pedramfar23_unified_approac_maxim_contin_dr_funct}, $\hat{\C{Q}}_t(\x)$ is an unbiased estimator of $\nabla \hat{f}_t(\x)$.
\footnote{When using a spherical estimator, it was shown in~\cite{flaxman2005online} that $\hat{f}$ is differentiable even when $f$ is not.
When using a sliced spherical estimator as we do here, differentiability of $\hat{f}$ is not proved in~\cite{pedramfar23_unified_approac_maxim_contin_dr_funct}.
However, their proof is based on the proof for the spherical case and therefore the results carry forward to show that $\hat{f}$ is differentiable.}
Hence we have $\hat{B} \in \op{Adv}_1^\t{o}(\hat{\BF{F}}, \frac{k}{\delta}B_0)$.
Using Equation~\ref{eq:semi-bandit-to-bandit:1} and the definition of the Meta-algorithm~\ref{alg:semi-bandit-to-bandit}, we see that the responses of the queries are the same between the game $(\C{A}, \hat{\C{B}})$ and $(\C{A}', \C{B})$.
It follows that the sequence of actions $(\x_1, \cdots, \x_T)$ in $(\C{A}, \hat{\C{B}})$ corresponds to the sequence of actions $(\x_1 + \delta \vv_1, \cdots, \x_T + \delta \vv_T)$ in $(\C{A}', \C{B})$.

Let $\uu \in \op{argmin}_{\uu \in \C{U}} \sum_{t = a}^b f_t(\uu_t)$ and $\hat{\uu} \in \op{argmin}_{\uu \in \hat{\C{U}}} \sum_{t = a}^b \hat{f_t}(\uu_t)$.
We have
\begin{align}
\C{R}_{\C{B}}^{\C{A}'}
- \C{R}_{\hat{\C{B}}}^{\C{A}}
&= \B{E}\left[ \sum_{t = a}^b f_t(\x_t + \delta \vv_t) - \sum_{t = a}^b f_t(\uu_t) \right]
- \B{E}\left[ \sum_{t = a}^b \hat{f_t}(\x_t) - \sum_{t = a}^b \hat{f_t}(\hat{\uu}_t) \right] \nonumber \\
&= \B{E}\left[ \left(\sum_{t = a}^b f_t(\x_t + \delta \vv_t) - \sum_{t = a}^b \hat{f_t}(\x_t) \right)
+ \left( \sum_{t = a}^b \hat{f_t}(\hat{\uu}_t) - \sum_{t = a}^b f_t(\uu_t) \right)\right].
\label{eq:semi-bandit-to-bandit:2}
\end{align}

First we consider the convex case.
By definition, $\hat{f_t}(\x_t)$ is an average of $f_t$ at points that are within $\delta$ distance of $\x$.
Therefore, according to Lemma~\ref{lem:conv-almost-lip} we have $| \hat{f_t}(\x_t) - f_t(\x_t) | \leq \frac{2 M_0 \delta}{\alpha}$.
By using Lemma~\ref{lem:conv-almost-lip} again for the pair $(\x_t, \x_t + \delta \vv_t)$, we see that
\begin{align}\label{eq:semi-bandit-to-bandit:3}
| f_t(\x_t + \delta \vv_t) - \hat{f_t}(\x_t) | 
\leq | f_t(\x_t + \delta \vv_t) - f_t(\x_t) | + | f_t(\x_t) - \hat{f_t}(\x_t) | 
\leq \frac{4 M_0 \delta}{\alpha}.
\end{align}
On the other hand, we have
\begin{align*}
\sum_{t = a}^b \hat{f_t}(\hat{\uu}_t) 
&= \min_{\hat{\w} \in \hat{\C{U}}} \sum_{t = a}^b \hat{f_t}(\hat{\w}_t) \\
&\leq \frac{2 M_0 \delta}{\alpha}T + \min_{\hat{\w} \in \hat{\C{U}}} \sum_{t = a}^b f_t(\hat{\w}_t) \tag{Lemma~\ref{lem:conv-almost-lip}} \\
&= \frac{2 M_0 \delta}{\alpha}T + \min_{\w \in \C{U}} \sum_{t = a}^b f_t\left(\left(1 - \frac{\alpha}{r} \right) \w_t + \frac{\alpha}{r} \BF{c} \right) 
\tag{Definition of $\hat{\C{U}}$}\\
&\leq \frac{2 M_0 \delta}{\alpha}T + \min_{\w \in \C{U}} \sum_{t = a}^b \left(\left(1 - \frac{\alpha}{r} \right) f_t(\w_t) + \frac{\alpha}{r} f_t(\BF{c}) \right) 
\tag{Convexity} \\
&\leq \frac{2 M_0 \delta}{\alpha}T + \min_{\w \in \C{U}} \sum_{t = a}^b \left(\left(1 - \frac{\alpha}{r} \right) f_t(\w_t) + \frac{\alpha}{r} M_0 \right) \\
&= \left( \frac{2 \delta}{\alpha} + \frac{\alpha}{r} \right) M_0 T + \left(1 - \frac{\alpha}{r} \right) \min_{\w \in \C{U}} \sum_{t = a}^b f_t(\w_t) \\
&\leq \left( \frac{2 \delta}{\alpha} + \frac{\alpha}{r} \right) M_0 T + \sum_{t = a}^b f_t(\uu_t).
\end{align*}
Putting these together, we see that
\begin{align*}
\C{R}_{\C{B}}^{\C{A}'}
- \C{R}_{\hat{\C{B}}}^{\C{A}}
&\leq \frac{4 M_0 \delta}{\alpha} T  + \left( \frac{2 \delta}{\alpha} + \frac{\alpha}{r} \right) M_0 T
= \left( \frac{6 \delta}{\alpha} + \frac{\alpha}{r} \right) M_0 T.
\end{align*}
Therefore, we have 
\begin{align*}
\C{R}_{\op{Adv}^\t{o}_0(\BF{F}, B_0)}^{\C{A}'}
&= \sup_{\C{B} \in \op{Adv}^\t{o}_0(\BF{F}, B_0)} \C{R}_{\C{B}, \C{Q}}^{\C{A}'} \\
&\leq \sup_{\C{B} \in \op{Adv}^\t{o}_0(\BF{F}, B_0)} \C{R}_{\hat{\C{B}}}^{\C{A}}
+ \left( \frac{6 \delta}{\alpha} + \frac{\alpha}{r} \right) M_0 T \\
&\leq \sup_{\C{B} \in \op{Adv}^\t{o}_0(\BF{F}, B_0)} \C{R}_{\hat{\C{B}}}^{\C{A}}
+ \left( \frac{6 \delta}{\alpha} + \frac{\alpha}{r} \right) M_0 T \\
&\leq \C{R}_{\op{Adv}^\t{o}_1(\hat{\BF{F}}, \frac{k}{\delta} B_0)}^{\C{A}}
+ \left( \frac{6 \delta}{\alpha} + \frac{\alpha}{r} \right) M_0 T.
\end{align*}

To prove the result in the Lipschitz case, we note that, according to Lemma~3 in~\cite{pedramfar23_unified_approac_maxim_contin_dr_funct}, we have
$| \hat{f_t}(\x_t) - f_t(\x_t) | \leq \delta M_1$.
By using Lipschitz property for the pair $(\x_t, \x_t + \delta \vv_t)$, we see that
\begin{align}\label{eq:semi-bandit-to-bandit:4}
| f_t(\x_t + \delta \vv_t) - \hat{f_t}(\x_t) | 
\leq | f_t(\x_t + \delta \vv_t) - f_t(\x_t) | + | f_t(\x_t) - \hat{f_t}(\x_t) | 
\leq 2 \delta M_1.
\end{align}
On the other hand, we have
\begin{align*}
\sum_{t = a}^b \hat{f_t}(\hat{\uu}_t) 
&= \min_{\hat{\w} \in \hat{\C{U}}} \sum_{t = a}^b \hat{f_t}(\hat{\w}_t) \\
&\leq \delta M_1 T + \min_{\hat{\w} \in \hat{\C{U}}} \sum_{t = a}^b f_t(\hat{\w}_t) \tag{Lemma~3 in~\cite{pedramfar23_unified_approac_maxim_contin_dr_funct}} \\
&= \delta M_1 T + \min_{\w \in \C{U}} \sum_{t = a}^b f_t\left(\left(1 - \frac{\alpha}{r} \right) \w_t + \frac{\alpha}{r} \BF{c} \right) 
\tag{Definition of $\hat{\C{U}}$}\\
&= \delta M_1 T + \min_{\w \in \C{K}} \sum_{t = a}^b f_t\left(\w_t + \frac{\alpha}{r} (\BF{c} - \x) \right) \\
&\leq \delta M_1 T + \min_{\w \in \C{U}} \sum_{t = a}^b \left( f_t(\w_t) + \frac{2 \alpha M_1 D}{r} \right) 
\tag{Lipschitz} \\
&= \left( 1 + \frac{2 D \alpha}{r \delta} \right) \delta M_1 T + \sum_{t = a}^b f_t(\uu_t)
\end{align*}
Therefore, using Equation~\ref{eq:semi-bandit-to-bandit:2}, we see that
\begin{align*}
\C{R}_{\C{B}}^{\C{A}'}
- \C{R}_{\hat{\C{B}}}^{\C{A}}
&\leq 2 \delta M_1 T  + \left( 1 + \frac{2 D \alpha}{r \delta} \right) \delta M_1 T 
= \left( 3 + \frac{2 D \alpha}{r \delta} \right) \delta M_1 T.
\end{align*}
Therefore, we have 
\begin{align*}
\C{R}_{\op{Adv}^\t{o}_0(\BF{F}, B_0)}^{\C{A}'}
&= \sup_{\C{B} \in \op{Adv}^\t{o}_0(\BF{F}, B_0)} \C{R}_{\C{B}}^{\C{A}'} \\
&\leq \sup_{\C{B} \in \op{Adv}^\t{o}_0(\BF{F}, B_0)} \C{R}_{\hat{\C{B}}}^{\C{A}}
+ \left( 3 + \frac{2 D \alpha}{r \delta} \right) \delta M_1 T \\
&\leq \sup_{\C{B} \in \op{Adv}^\t{o}_0(\BF{F}, B_0)} \C{R}_{\hat{\C{B}}}^{\C{A}}
+ \left( 3 + \frac{2 D \alpha}{r \delta} \right) \delta M_1 T \\
&\leq \C{R}_{\op{Adv}^\t{o}_1(\hat{\BF{F}}, \frac{k}{\delta} B_0)}^{\C{A}}
+ \left( 3 + \frac{2 D \alpha}{r \delta} \right) \delta M_1 T.
\end{align*}
\end{proof}

\begin{proof}[Proof of Theorem~\ref{thm:first-order-to-zero-order}]

The proof of the bounds for this case is similar to the previous case.
As before, we see that the responses of the queries are the same between the game $(\C{A}, \hat{\C{B}})$ and $(\C{A}', \C{B})$.
It follows from the description of Meta-algorithm~\ref{alg:first-order-to-zeroth-order} that the sequence of actions $(\x_1, \cdots, \x_T)$ in $(\C{A}, \hat{\C{B}})$ corresponds to the same sequence of actions in $(\C{A}', \C{B})$.

Let $\uu \in \op{argmin}_{\uu \in \C{U}} \sum_{t = a}^b f_t(\uu_t)$ and $\hat{\uu} \in \op{argmin}_{\uu \in \hat{\C{U}}} \sum_{t = a}^b \hat{f_t}(\uu_t)$.
We have
\begin{align}
\C{R}_{\C{B}}^{\C{A}'}
- \C{R}_{\hat{\C{B}}}^{\C{A}}
&= \B{E}\left[ \sum_{t = a}^b f_t(\x_t) - \sum_{t = a}^b f_t(\uu_t) \right]
- \B{E}\left[ \sum_{t = a}^b \hat{f_t}(\x_t) - \sum_{t = a}^b \hat{f_t}(\hat{\uu}_t) \right] \nonumber \\
&= \B{E}\left[ \left(\sum_{t = a}^b f_t(\x_t) - \sum_{t = a}^b \hat{f_t}(\x_t) \right)
+ \left( \sum_{t = a}^b \hat{f_t}(\hat{\uu}_t) - \sum_{t = a}^b f_t(\uu_t) \right)\right].
\label{eq:semi-bandit-to-bandit:5}
\end{align}
To obtain the same bound as before, we note that in the convex case, instead of Inequality~\ref{eq:semi-bandit-to-bandit:3}, we have
\begin{align}
| f_t(\x_t) - \hat{f_t}(\x_t) | 
\leq \frac{2 M_0 \delta}{\alpha}
< \frac{4 M_0 \delta}{\alpha},
\end{align}
and, in the Lipschitz case, instead of Inequality~\ref{eq:semi-bandit-to-bandit:4}, we have
\begin{align*}
| f_t(\x_t) - \hat{f_t}(\x_t) | 
\leq \delta M_1
< 2 \delta M_1.
\end{align*}
The rest of the proof follows verbatim.
\end{proof}



\section{Proof of Theorem~\ref{thm:first-order-to-det-zero-order}}\label{app:first-order-to-det-zero-order}

This proof is similar to the proofs of Theorems~\ref{thm:first-order-to-zero-order} and~\ref{thm:semi-bandit-to-bandit}.

\begin{proof}
Note that any realized adversary $\C{B} \in \op{Adv}^\t{o}_0(\BF{F})$ may be represented as a sequence of functions $(f_1, \cdots, f_T)$.
For such realized adversary $\C{B}$, we define $\hat{B}$ to be the realized adversary corresponding to $(\hat{f}_1, \cdots, \hat{f}_T)$ with the stochastic gradient oracles
\begin{align}\label{eq:first-order-to-det-zero-order:1}
\hat{\C{Q}}_t(\x) := \frac{k}{2\delta} \left( f_t(\x + \delta \vv) - f_t(\x - \delta \vv) \right) \vv,
\end{align}
where $\vv$ is a random vector, taking its values uniformly from $\B{S}^1 \cap \C{L}_0 = \B{S}^1 \cap (\op{aff}(\C{K}) - \z)$, for any $\z \in \C{K}$ and $k = \op{dim}(\C{L}_0)$.
Since $\C{Q}_t$ is a stochastic value oracle for $f_t$, according to Lemma~5 in~\cite{pedramfar23_unified_approac_maxim_contin_dr_funct}, $\hat{\C{Q}}_t(\x)$ is an unbiased estimator of $\nabla \hat{f}_t(\x)$.
Moreover, we have
\begin{align*}
\left\| \frac{k}{2\delta} \left( f_t(\x + \delta \vv) - f_t(\x - \delta \vv) \right) \vv \right\|
\leq \frac{k}{2\delta} M_1 \| (\x + \delta \vv) - (\x - \delta \vv) \|
\leq k M_1.
\end{align*}
Hence we have $\hat{B} \in \op{Adv}_1^\t{o}(\hat{\BF{F}}, k M_1)$.
Using Equation~\ref{eq:first-order-to-det-zero-order:1} and the definition of the Meta-algorithm~\ref{alg:first-order-to-det-zeroth-order}, we see that the responses of the queries are the same between the game $(\C{A}, \hat{\C{B}})$ and $(\C{A}', \C{B})$.
It follows that the sequence of actions $(\x_1, \cdots, \x_T)$ in $(\C{A}, \hat{\C{B}})$ corresponds to the same sequence of actions in $(\C{A}', \C{B})$.

Let $\uu \in \op{argmin}_{\uu \in \C{U}} \sum_{t = a}^b f_t(\uu_t)$ and $\hat{\uu} \in \op{argmin}_{\uu \in \hat{\C{U}}} \sum_{t = a}^b \hat{f_t}(\uu_t)$.
We have
\begin{align}
\C{R}_{\C{B}}^{\C{A}'}
- \C{R}_{\hat{\C{B}}}^{\C{A}}
&= \B{E}\left[ \sum_{t = a}^b f_t(\x_t) - \sum_{t = a}^b f_t(\uu_t) \right]
- \B{E}\left[ \sum_{t = a}^b \hat{f_t}(\x_t) - \sum_{t = a}^b \hat{f_t}(\hat{\uu}_t) \right] \nonumber \\
&= \B{E}\left[ \left(\sum_{t = a}^b f_t(\x_t) - \sum_{t = a}^b \hat{f_t}(\x_t) \right)
+ \left( \sum_{t = a}^b \hat{f_t}(\hat{\uu}_t) - \sum_{t = a}^b f_t(\uu_t) \right)\right].
\label{eq:first-order-to-det-zero-order:2}
\end{align}
According to Lemma~3 in~\cite{pedramfar23_unified_approac_maxim_contin_dr_funct}, we have
$| \hat{f_t}(\x_t) - f_t(\x_t) | \leq \delta M_1$.
On the other hand, we have
\begin{align*}
\sum_{t = a}^b \hat{f_t}(\hat{\uu}_t) 
&= \min_{\hat{\w} \in \hat{\C{U}}} \sum_{t = a}^b \hat{f_t}(\hat{\w}_t) \\
&\leq \delta M_1 T + \min_{\hat{\w} \in \hat{\C{U}}} \sum_{t = a}^b f_t(\hat{\w}_t) \tag{Lemma~3 in~\cite{pedramfar23_unified_approac_maxim_contin_dr_funct}} \\
&= \delta M_1 T + \min_{\w \in \C{U}} \sum_{t = a}^b f_t\left(\left(1 - \frac{\delta}{r} \right) \w_t + \frac{\delta}{r} \BF{c} \right) 
\tag{Definition of $\hat{\C{U}}$}\\
&= \delta M_1 T + \min_{\w \in \C{K}} \sum_{t = a}^b f_t\left(\w_t + \frac{\delta}{r} (\BF{c} - \w_t) \right) \\
&\leq \delta M_1 T + \min_{\w \in \C{U}} \sum_{t = a}^b \left( f_t(\w_t) + \frac{2 \delta M_1 D}{r} \right) 
\tag{Lipschitz} \\
&= \left( 1 + \frac{2 D}{r} \right) \delta M_1 T + \sum_{t = a}^b f_t(\uu_t)
\end{align*}
Therefore, using Equation~\ref{eq:first-order-to-det-zero-order:2}, we see that
\begin{align*}
\C{R}_{\C{B}}^{\C{A}'}
- \C{R}_{\hat{\C{B}}}^{\C{A}}
&\leq \delta M_1 T  + \left( 1 + \frac{2 D}{r} \right) \delta M_1 T 
= \left( 2 + \frac{2 D}{r} \right) \delta M_1 T.
\end{align*}
Therefore, we have 
\begin{align*}
\C{R}_{\op{Adv}^\t{o}_0(\BF{F})}^{\C{A}'}
&= \sup_{\C{B} \in \op{Adv}^\t{o}_0(\BF{F})} \C{R}_{\C{B}}^{\C{A}'} \\
&\leq \sup_{\C{B} \in \op{Adv}^\t{o}_0(\BF{F})} \C{R}_{\hat{\C{B}}}^{\C{A}}
+ \left( 2 + \frac{2 D}{r} \right) \delta M_1 T \\
&\leq \sup_{\C{B} \in \op{Adv}^\t{o}_0(\BF{F})} \C{R}_{\hat{\C{B}}}^{\C{A}}
+ \left( 2 + \frac{2 D}{r} \right) \delta M_1 T \\
&\leq \C{R}_{\op{Adv}^\t{o}_1(\hat{\BF{F}}, k M_1)}^{\C{A}}
+ \left( 2 + \frac{2 D}{r} \right) \delta M_1 T.
\qedhere
\end{align*}
\end{proof}